\documentclass{article}

\usepackage[preprint]{neurips_2026}

\usepackage[utf8]{inputenc} 
\usepackage[T1]{fontenc}    
\usepackage{hyperref}       
\usepackage{url}            
\usepackage{booktabs}       
\usepackage{amsfonts}       
\usepackage{nicefrac}       
\usepackage{microtype}      
\usepackage{xcolor}         
\usepackage{multirow}
\usepackage{tabularx}
\usepackage[table]{xcolor}
\usepackage{algorithm}
\usepackage{algorithmic}
\usepackage{amsmath}
\usepackage[most]{tcolorbox}
\usepackage{xcolor}
\usepackage{amsthm}
\usepackage{amssymb}
\usepackage{enumitem}
\usepackage{wrapfig}

\definecolor{mygray}{RGB}{210,210,210}
\usepackage{arydshln}
\newtheorem{theorem}{Theorem}[section]
\newtheorem{lemma}[theorem]{Lemma}

\newcommand{\figref}[1]{Figure \ref{#1}}

\newcommand{\secref}[1]{Section \ref{#1}}

\newcommand{\appendixref}[1]{Appendix \ref{#1}}

\newcommand{\theoremref}[1]{Theorem \ref{#1}}
\newcommand{\lemmaref}[1]{Lemma \ref{#1}}

\theoremstyle{remark}

\definecolor{takeawaybg}{RGB}{235,242,250}

\newtcolorbox{takeawaybox}{
    enhanced,
    colback=takeawaybg,
    colframe=black,
    boxrule=0.6pt,
    arc=2pt,
    outer arc=2pt,
    left=3pt,
    right=3pt,
    top=3pt,
    bottom=3pt,
    width=\linewidth,
    before skip=4pt,
    after skip=4pt,
}

\title{When Self-Belief Misleads: Active Label Acquisition for Reinforcement Learning with Verifiable Rewards}

\author{%
\begin{tabular}{c}
Li Wang$^{1}$\thanks{Equal contribution.} \quad
Xiaodong Lu$^{1}$\footnotemark[1] \quad
Xiaohan Wang$^{1}$\thanks{Corresponding authors. Emails:
\texttt{wangxiaohan17@meituan.com},
\texttt{tianhao.peng@ntu.edu.sg},
\texttt{yinguojun02@meituan.com}.} \quad
Yikun Ban$^{2}$ \\
Jiajun Chai$^{1}$ \quad
Wei Lin$^{1}$ \quad
Tianhao Peng$^{3}$\footnotemark[2] \quad
Guojun Yin$^{1}$\footnotemark[2] \\
\\
$^{1}$Meituan \quad
$^{2}$Beihang University \quad
$^{3}$Nanyang Technological University
\end{tabular}%
}

\begin{document}

\maketitle

\begin{abstract}
  Large Language Models (LLMs) have achieved remarkable advancements in reasoning capabilities empowered by Reinforcement Learning with Verifiable Rewards (RLVR). Nonetheless, RLVR intrinsically relies on ground-truth labels for reward computation, the acquisition of which is often prohibitively expensive in real-world scenarios. While unsupervised RLVR paradigms attempt to circumvent this by training on pseudo-labels, they are notoriously susceptible to training collapse. Moreover, different samples often exhibit varying annotation values. In this paper, we propose Reinforcement Learning with Active Verifiable Rewards (RLAVR), which actively acquires ground-truth labels for a small set of selected samples and integrates them with pseudo-labels, thereby stabilizing training dynamics and improving performance under limited annotation budgets. To identify valuable samples, we propose the Corrective Advantage Gap (CAG) metric and analyze the sample-level supervision value. Building on this, we introduce \underline{C}orrection-\underline{A}ware \underline{R}eliability \underline{E}stimation for RLAVR (CARE), which translates the oracle CAG criterion into a practical pre-query acquisition policy to substantially improve training stability. Extensive experiments across diverse domains, model families, and model scales demonstrate the effectiveness and generality of our approach. Our code is available at \url{https://github.com/Lumina04/CARE}.
\end{abstract}

\section{Introduction}
From DeepSeek-R1 to the GPT-O series, Reinforcement Learning with Verifiable Rewards (RLVR) \cite{wang2026survey,guan2025rstar,yang2025code,chen2025acereason} has emerged as a critical method for advancing the reasoning capabilities of Large Language Models (LLMs). 
By utilizing stable feedback from the verifiers (e.g., rule-based verifiers), RLVR enables the LLM policy to continuously refine its behavior through an iterative cycle of exploration and policy optimization. Although effective, RLVR faces a fundamental challenge in its dependence on ground-truth labels to compute rewards, such as standard answers for math problems or test cases for coding tasks. These labels are often costly to obtain, and the resulting reliance on external supervision constrains the policy’s ability to self-improve and explore autonomously.

Unsupervised RLVR~\cite{zuo2025ttrl, he2026far} takes a significant step toward eliminating the reliance on supervised data. It leverages intrinsic model signals, such as majority voting~\cite{zhang2025co, liao2026tool} and self-consistency~\cite{zhao2025learning, ghimire2026prism}, as pseudo-labels for reward computation. While promising, recent studies suggest that the model's internal beliefs are not always reliable~\cite{he2026far, tan2025diagnosing, zhou2025evolving}. Pure internal signal-based supervision may form a dead loop that reinforces erroneous beliefs and ultimately leads to training collapse (shown in \secref{sec:reliability_of_unsupervised_RLVR}). This failure mode suggests the need for external supervision, yet the supervision must be used sparingly: acquiring ground-truth labels for all samples would lead to an expensive annotation cost. Consequently, a natural question arises: 

\emph{How can we actively acquire sparse ground-truth labels for a small set of critical samples and combine them with pseudo-labels to provide reliable training signals?}

To answer this question, this paper makes the following three technical contributions.

\textbf{New Problem and Justification.} Motivated by this need, we introduce Reinforcement Learning with Active Verifiable Rewards (RLAVR), to our knowledge, the first active label-acquisition problem for RLVR. Specifically, we formulate RLAVR as a budget-constrained label assignment problem: under a limited annotation budget, the learner actively selects which samples to query for ground-truth annotations and combines the verified labels with pseudo-labels to construct mixed training signals. To justify this problem setting, we examine the supervision value of different samples, measured by the performance gain obtained from correcting erroneous pseudo-labels with ground-truth labels. Our analysis reveals that samples exhibit highly heterogeneous supervision value, suggesting that sparse annotations should be allocated selectively rather than uniformly.

\textbf{CAG-Guided Acquisition Framework.} To solve the proposed problem, we first propose the Corrective Advantage Gap (CAG), a metric that characterizes the discrepancy between the advantage vectors induced by ground-truth labels and pseudo-labels. Empirically, CAG faithfully reflects supervision value: across different settings, selecting samples with larger CAG consistently yields stronger performance gains than alternative selection strategies (shown in \secref{sec:supervision_value_analysis}). However, CAG requires ground-truth labels and therefore cannot be directly computed before annotation. Therefore, we further propose \underline{C}orrection-\underline{A}ware \underline{R}eliability \underline{E}stimation for RLAVR (CARE), a method for introducing sparse supervision into RLAVR by predicting sample reliability and supervision value before querying. Specifically, CARE trains a dedicated two-stage classifier network: the first-stage classifier estimates pseudo-label reliability and retains reliable pseudo-label signals, while the second-stage classifier predicts expected CAG for unreliable samples.  Guided by the classifiers, CARE prioritizes samples whose pseudo-labels are likely to induce high supervision value, thereby fully exploiting reliable pseudo-labeled samples while suppressing erroneous internal reward signals.

\textbf{Theoretical and Empirical Validation.} Theoretically, we justify CAG by deriving a CAG-dependent upper bound on gradient alignment: larger advantage deviations imply a smaller guaranteed cosine alignment between ground-truth-induced and pseudo-label-induced GRPO gradients. This explains why high-CAG samples can introduce more misaligned optimization signals if left uncorrected. Empirically, with only a 20\% annotation budget, CARE consistently improves over TTRL across different model scales and families on both mathematical and logical reasoning tasks, while substantially narrowing the gap to fully supervised RLVR. Further ablations validate the effectiveness of reliability estimation, CAG prediction, and sample dropping.



\section{Related Work}
\subsection{Unsupervised and Semi-supervised RLVR}
The success of DeepSeek-R1~\cite{guo2025deepseek} has demonstrated the substantial potential of RLVR~\cite{shao2024deepseekmath} for enhancing the reasoning capabilities of LLMs through reinforcement learning. Building on this insight, a growing body of work~\cite{lu2026contextual,lin2025awpo,lin2025rest} has explored the RLVR paradigm and achieved promising results. However, standard RLVR relies on ground-truth labels to compute reward signals for training, whereas in many practical scenarios such labels are difficult or costly to obtain. To enable training in unlabeled settings, numerous studies have investigated unsupervised RLVR, where pseudo-labels are constructed from signals such as majority voting~\cite{zuo2025ttrl,wang2025beyond}, self-certainty~\cite{zhao2025learning,ghimire2026prism}, and entropy~\cite{zhang2025right,agarwal2025unreasonable,prabhudesai2025maximizing}. Some methods further introduce more refined designs by modeling answer distributions~\cite{yu2025restrain} or exploiting discrepancies between correct and incorrect reasoning trajectories~\cite{zhang2025consistent}, thereby yielding additional performance gains. Concurrently, several works have noted that fully label-free paradigms are prone to training instability or even collapse~\cite{he2026far}, and have mitigated this by designing specific reward mechanisms~\cite{tan2025diagnosing,zhou2025evolving} or trained verifiers~\cite{pan2026coverrl}. As an alternative approach, semi-supervised RLVR methods~\cite{yang2025trapo,luo2026memreward} incorporate a limited amount of ground-truth supervision to guide sample selection or pseudo-label generation, leading to more effective training. Although RLAVR also involves mixed training with pseudo-labels and ground-truth labels, its focus lies in actively selecting valuable samples for annotation. Semi-supervised RLVR methods typically assume a given labeled set and focus on exploiting existing labeled and unlabeled data. In contrast, RLAVR focuses on deciding which unlabeled samples to annotate and combining their ground-truth labels with pseudo-labels to improve training under limited annotation budgets. 


\subsection{Active Learning}
Active learning~\cite{aggarwal2014active} aims to reduce annotation costs by selecting the most informative samples for labeling, thereby maximizing model improvement under limited supervision. To this end, a variety of acquisition strategies have been explored in the literature, including uncertainty sampling~\cite{wang2014new,geifman2017deep}, query-by-committee~\cite{kee2018query}, expected model change~\cite{cai2013maximizing}, expected error reduction~\cite{mussmann2022active}, and density-aware methods~\cite{wang2021density}. In the LLM era, where high-quality human annotation is substantially more expensive, active learning has become increasingly attractive as a mechanism for improving data efficiency while maintaining or even enhancing downstream performance. Recent studies have incorporated active selection into Supervised Fine-Tuning (SFT)~\cite{bayer2026activellm}, alignment and post-training pipelines, particularly in Reinforcement Learning from Human Feedback (RLHF)~\cite{bai2022training}, where samples are selected based on signals such as response redundancy~\cite{feng2025duo} or model uncertainty~\cite{duan2025efficient} to improve preference data collection efficiency~\cite{liu2024dual,mehta2023sample}. Related ideas have also been explored in post-training settings, for example, by actively acquiring labels for Direct Preference Optimization (DPO)~\cite{rafailov2023direct,ji2024reinforcement} or by using active data filtering to improve the stability of reinforcement learning~\cite{yi2026learn}. Unlike prior work, we target RLAVR by integrating active learning into online RL training, enabling selective label acquisition and continuous improvement of model capabilities.
\section{Preliminary}
\subsection{Group Relative Policy Optimization}
Group Relative Policy Optimization (GRPO)~\cite{shao2024deepseekmath} is a popular policy optimization method for RLVR, which first introduces the group-based advantage estimator. Specifically, for each prompt $x_i$, the current policy samples a group of $G$ rollouts, denoted by $\mathcal{O}_i=\{o_{i,g}\}_{g=1}^G$, and each rollout $o_{i,g}$ is assigned a scalar reward $r_{i,g}$ by a rule-based verifier. GRPO computes group-relative advantages by normalizing rewards within each rollout group,
\begin{equation}
A_{i,g}=\frac{r_{i,g}-\mu_i(r)}{\sigma_i(r)+\epsilon},
\label{eq:grpo_adv}
\end{equation}
where $\mu_i(r)$ and $\sigma_i(r)$ denote the mean and standard deviation of $\{r_{i,g}\}_{g=1}^G$, respectively, and $\epsilon$ is a small constant for numerical stability. Since rewards are defined at the rollout level, the same advantage $A_{i,g}$ is shared across all tokens in $o_{i,g}$. The policy is then optimized by
\begin{equation}
\begin{aligned}
\mathcal{J}_{\mathrm{GRPO}}(\theta)
=&\;
\mathbb{E}_{x_i,\mathcal{O}_i}
\Bigg[
\frac{1}{G}\sum_{g=1}^G \frac{1}{|o_{i,g}|}\sum_{t=1}^{|o_{i,g}|}
\Big(
\min\big(
\rho_{i,g,t}A_{i,g}, \mathrm{clip}(\rho_{i,g,t},1-\delta,1+\delta)A_{i,g}
\big)]
\Big)
\Bigg],
\end{aligned}
\label{eq:grpo_obj}
\end{equation}
where $\rho_{i,g,t}=\pi_\theta(o_{i,g,t}\mid x_i,o_{i,g,<t})/\pi_{\theta_{\mathrm{old}}}(o_{i,g,t}\mid x_i,o_{i,g,<t})$ is the token-level importance ratio, and $\delta$ controls the clipping range.

\subsection{Reinforcement Learning with Active Verifiable Rewards} \label{sec:problem_formulation}
We introduce Reinforcement Learning with Active Verifiable Rewards (RLAVR). The studied problem can be formulated as a budget-constraint label assignment problem. Formally, assume the training proceeds for $T$ rounds. At each round $t$, a batch of prompts $B_t = \{x_1^t, \ldots, x_n^t\}$ is first sampled. Then, for each prompt $x_i^t$, the policy $\pi_{\theta_t}$ generates a group of responses $\{o_{i,g}^t\}_{g=1}^G$, resulting in a set of prompt-response pairs $C_t = \{(x_i^t, \{o_{i,g}^t\}_{g=1}^G)\}_{i=1}^n$. We hope to select a subset of informative samples $I_t^{L} \subseteq [n]$ for label querying and obtain ground-truth labels $y_i^t$ for $i \in I_t^{L}$, while for the remaining samples $i \notin I_t^{L}$, pseudo-labels $\hat{y}_i^t = \mathcal{A}(x_i^t,\{o_{i,g}^t\})$ are constructed via an aggregation function. All labeled samples, including both queried and pseudo-labeled ones, are then used to form a training batch for reward and advantage computation, based on which the policy is updated from $\pi_{\theta_t}$ to $\pi_{\theta_{t+1}}$ according to Eq.~\eqref{eq:grpo_obj}. Our goal is to maximize the performance of the best checkpoint over $T$ rounds under a predefined labeling budget constraint $p$:
\begin{align}
    \max_{\{I_t^{L}\}_{t=1}^T} \ \max_{t \in [T]} \ \mathrm{Perf}(\pi_{\theta_t})
    \quad \text{s.t.} \quad
    \frac{1}{nT}\sum_{t=1}^T |I_t^{L}| \le p,
\end{align}
where $\mathrm{Perf}(\pi_{\theta_t})$ denotes the evaluation performance of checkpoint $\pi_{\theta_t}$. It is worth noting that RLAVR differs from semi-supervised RLVR, where labeled samples are predefined. Instead, the goal is to actively select samples for annotation during online training.
\section{Preliminary Analysis of RLAVR}

\subsection{Reliability Analysis for Unsupervised RLVR Training} \label{sec:reliability_of_unsupervised_RLVR}
To assess the reliability of the model's internal beliefs, we compared the GRPO curves trained with true labels, pseudo-labels generated through majority voting, and pseudo-labels with errors masked (detailed experimental settings and additional results can be found in Appendix~\ref{appendix:more_collapse_pre_exp}.) As shown in \figref{fig:pre_exp1}, the performance of the model trained with pseudo-labels rapidly deteriorates after an initial phase of improvement, aligning with previous studies \cite{he2026far, tan2025diagnosing, zhou2025evolving} that the model's internal signals are not always reliable. Notably, masking erroneous pseudo-labels stabilizes the model's performance, further attributing the primary cause of training collapse to the misguidance of incorrect internal beliefs. The change in pseudo-label accuracy exposes the underlying mechanism: before the sharp performance decline, the accuracy of pseudo-labels continually decreases, forming a dead loop that keeps reinforcing incorrect internal beliefs, ultimately leading to training collapse.

\vspace{0.25cm}
\begin{takeawaybox}
\textbf{Takeaway :} Unsupervised RLVR driven by internal beliefs is not always reliable. Erroneous beliefs can be continuously reinforced, causing the policy to deviate from the correct optimization direction and eventually collapse.
\end{takeawaybox}

\begin{figure}[t]
    \centering
    \begin{minipage}[t]{0.32\linewidth}
        \centering
        \includegraphics[width=\linewidth]{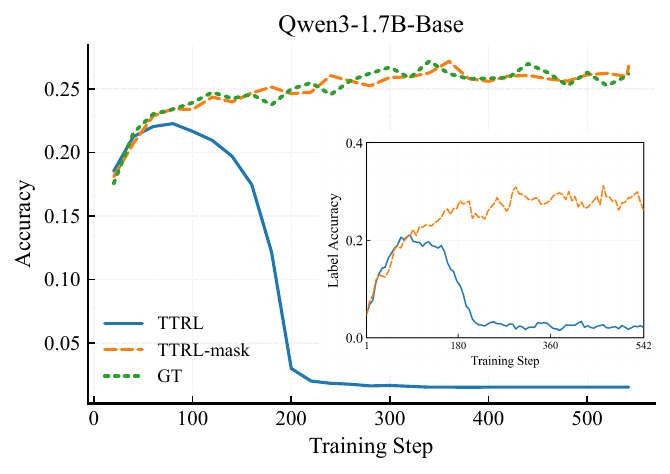}
    \end{minipage}
    \begin{minipage}[t]{0.32\linewidth}
        \centering
        \includegraphics[width=\linewidth]{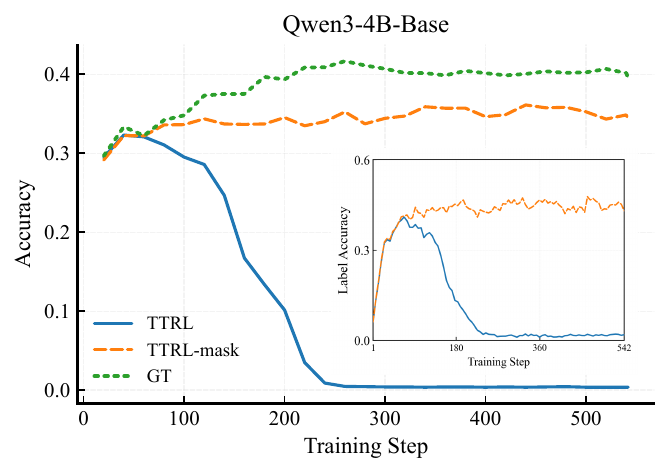}
    \end{minipage}
    \begin{minipage}[t]{0.32\linewidth}
        \centering
        \includegraphics[width=\linewidth]{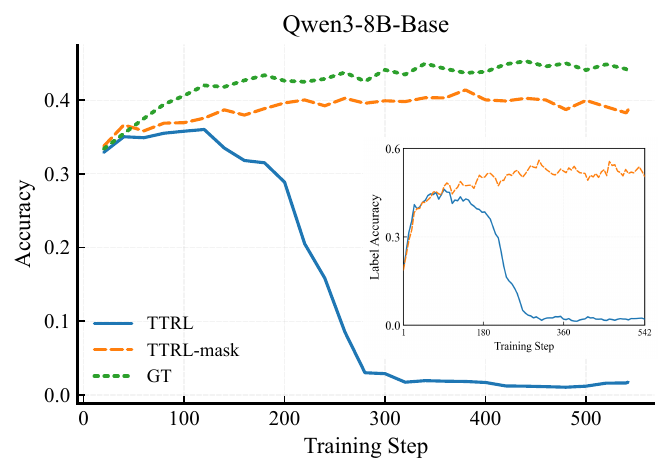}
    \end{minipage}
    \caption{Training dynamics across three Qwen3 models on math tasks.}
    \label{fig:pre_exp1}
\end{figure}

\begin{figure}[t]
    \centering
    \begin{minipage}[t]{0.32\linewidth}
        \centering
        \includegraphics[width=\linewidth]{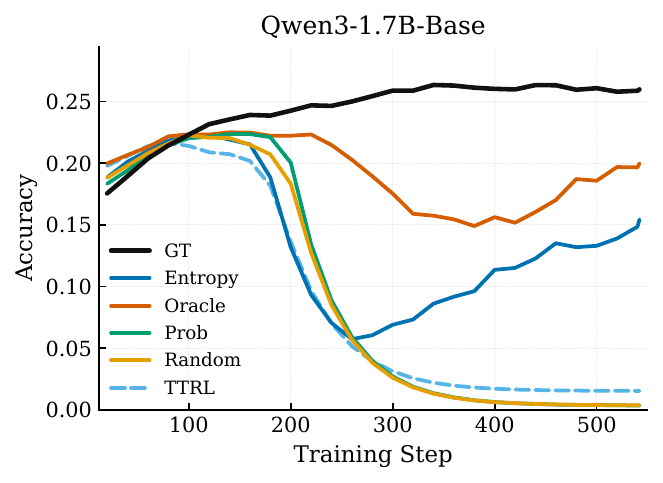}
    \end{minipage}
    \begin{minipage}[t]{0.32\linewidth}
        \centering
        \includegraphics[width=\linewidth]{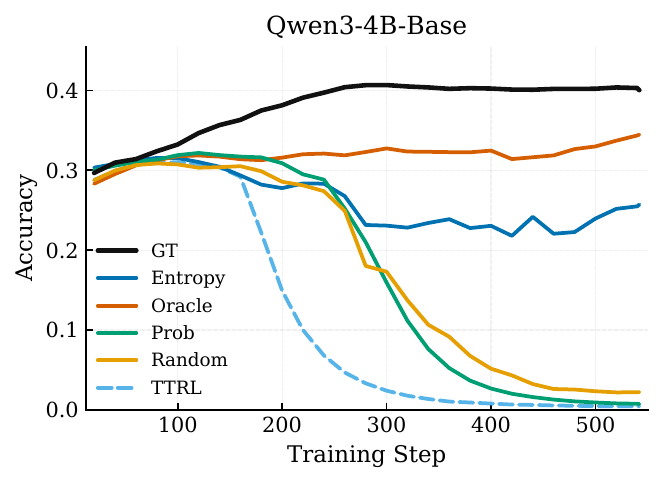}
    \end{minipage}
    \begin{minipage}[t]{0.32\linewidth}
        \centering
        \includegraphics[width=\linewidth]{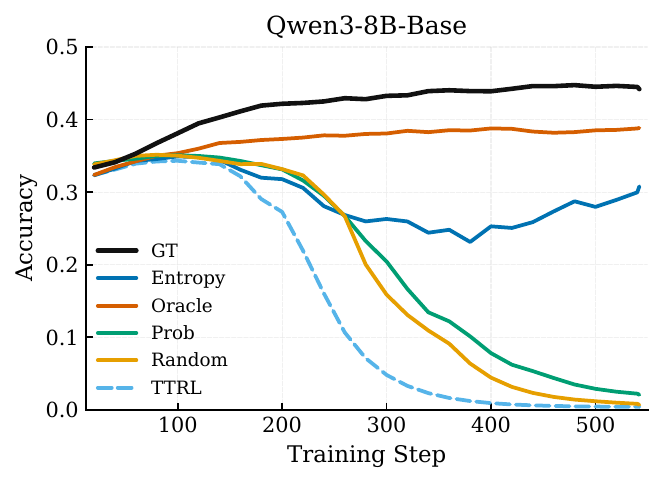}
    \end{minipage}
    \caption{Comparison of label acquisition strategies across three Qwen3 models on math tasks.}
    \label{fig:pre_exp2}
\end{figure}


\subsection{Supervision Value Analysis of Erroneous Samples}
\label{sec:supervision_value_analysis}
The preceding results indicate that erroneous pseudo-labels are the main driver of training collapse. As a preliminary study, we further ask whether correcting different erroneous pseudo-labels leads to similar performance gains, which we interpret as their supervision value. Specifically, given a limited annotation budget, we compare the effect of correcting different subsets of incorrect pseudo-labels.

Intuitively, an erroneous sample should have a high supervision value if its pseudo-label induces a learning signal that differs substantially from the one induced by the ground-truth label. Following this idea, we quantify supervision value by the advantage difference between ground-truth labels and pseudo-labels. Specifically, for the \(G\) responses to a prompt \(x_i\), let \(A_i=(A_{i,1},\dots,A_{i,G})\) and \(\tilde A_i=(\tilde A_{i,1},\dots,\tilde A_{i,G})\) denote the advantage vectors computed with ground-truth and pseudo-labels, respectively, as per Eq.~\eqref{eq:grpo_adv}. We define the \emph{Corrective Advantage Gap} (CAG) score as
\begin{equation}
\mathrm{s}_i = \|A_i-\tilde A_i\|_2.
\label{eq:cag}
\end{equation}

As shown in \figref{fig:pre_exp2}, assigning ground-truth labels by CAG (Oracle) yields the largest performance gain among the compared selection criteria, including random and response-distribution-based selection. Notably, the other strategies still suffer from training collapse, suggesting that CAG more faithfully identifies erroneous pseudo-labels with high supervision value. Surprisingly, even when the remaining pseudo-labeled samples are masked out and training uses only the acquired ground-truth labels, the CAG-selected samples still achieve strong performance gains (see Appendix~\ref{appendix:pre_mask_exp} for details and more results). This further shows that high-CAG samples are not only useful for removing harmful pseudo-labels, but also consistently provide informative supervision signals for learning.

\begin{takeawaybox}
\textbf{Takeaway:} Different samples carry substantially different supervision value, and CAG provides a reliable signal for identifying samples whose annotation yields larger learning gains.
\end{takeawaybox}


\section{Methodology}
\label{sec:method}


Based on the preceding analysis, CAG provides an effective criterion for measuring the supervision value of correcting a pseudo-label. However, it is an oracle quantity at acquisition time, since computing CAG requires the ground-truth label. We therefore propose \underline{C}orrection-\underline{A}ware
\underline{R}eliability \underline{E}stimation for RLAVR (CARE), which turns the oracle CAG criterion into a practical pre-query acquisition policy, as illustrated in Figure~\ref{fig:care_method}. CARE estimates pseudo-label reliability and expected supervision value from rollout-derived signals, thereby retaining reliable pseudo-labeled samples for unsupervised optimization while querying unreliable samples with high supervision value. The full algorithm is summarized in Algorithm~\ref{alg:care}.

\begin{figure}[t]
    \centering
    \includegraphics[width=1.0\linewidth]{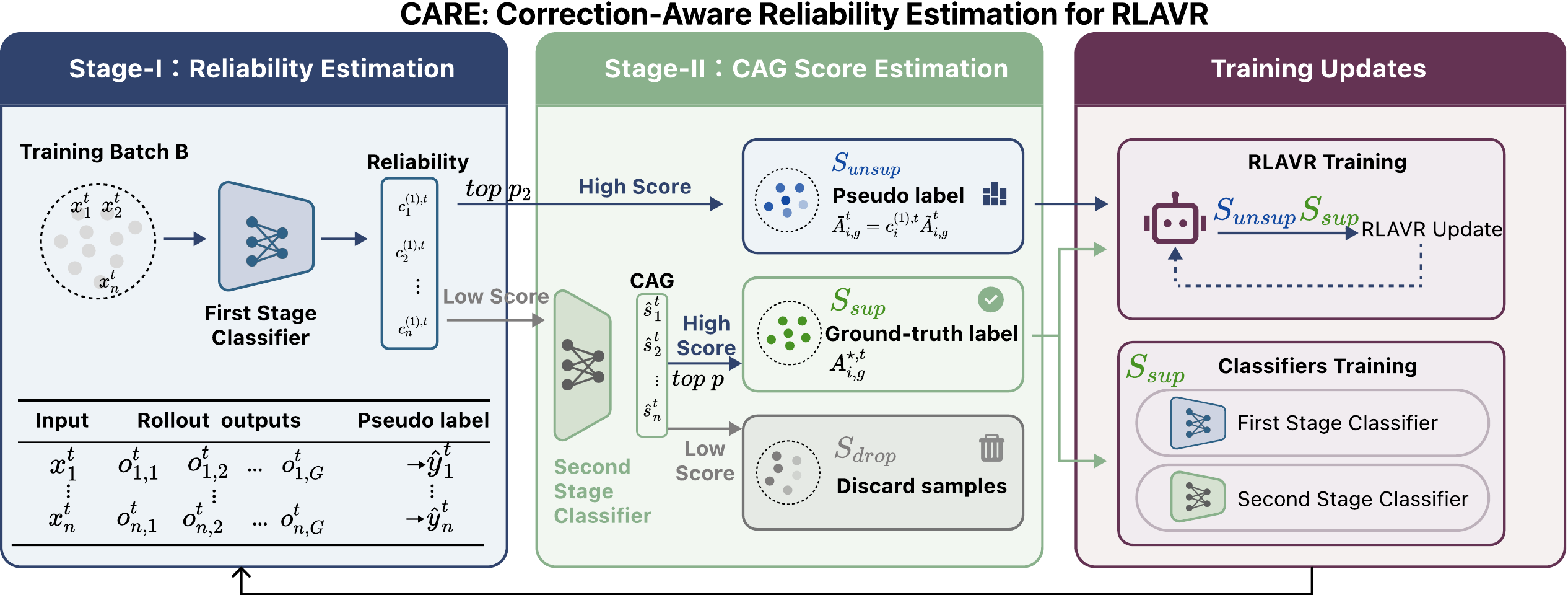}
    \caption{Overall pipeline of CARE. The Stage-I classifier predicts sample reliability, the Stage-II classifier predicts CAG to select informative samples for annotation, and the resulting labeled and pseudo-labeled data are used to jointly train the LLM and the classifiers.}
    \label{fig:care_method}
\end{figure}

\subsection{Overall Framework}
\label{sec:care_framework}

Although CAG shows promise as an annotation-value criterion, applying it uniformly to all samples ignores the hierarchical nature of the budgeted decision. Specifically, when a pseudo-label is correct, it induces the same advantage vector as the ground-truth label and thus has zero CAG; such samples should be retained for low-cost
unsupervised training rather than queried. CAG ranking is meaningful only within the unreliable region: among samples with incorrect pseudo-labels, high-CAG ones are worth annotating because their ground-truth labels would bring a large performance gain, while low-CAG ones can be discarded due to their wrong learning signals. CARE therefore decomposes acquisition into two conditional steps. Stage-I estimates pseudo-label reliability to retain reliable samples, and Stage-II ranks the predicted unreliable samples by expected CAG for annotation.

\subsubsection{Cascaded Classifiers}
\label{sec:cascaded_classifiers}

The cascade implements the conditional acquisition rule described above. At round $t$, for each prompt $x_i^t\in\mathcal{B}^t$ with rollouts $\{o_{i,g}^t\}_{g=1}^G$, CARE first builds a rollout-derived representation $\phi_i^t \in \mathbb R^d$ by
\begin{equation}
\phi_i^t = E(x_i^t,\{o_{i,g}^t\}_{g=1}^G),
\end{equation}
where $E(\cdot)$ is a lightweight feature encoder detailed in Section~\ref{sec:detail_implement}. CARE
then applies two neural classifiers: a reliability classifier
$f_{\psi_1^t}^{\mathrm{rel}}$ and a supervision-value classifier
$f_{\psi_2^t}^{\mathrm{cag}}$. Together, they partition $\mathcal{B}^t$ into
three disjoint subsets: reliable pseudo-labeled samples
$\mathcal{S}_{\mathrm{unsup}}^t$, unreliable high-value samples selected for
annotation $\mathcal{S}_{\mathrm{sup}}^t$, and unreliable low-value samples
discarded from policy optimization $\mathcal{S}_{\mathrm{drop}}^t$. In Stage-I, the reliability classifier outputs
$c_i^{(1),t}=f_{\psi_1^t}^{\mathrm{rel}}(\phi_i^t)\in[0,1]$, interpreted as the probability that the pseudo-label $\hat{y}_i^t$ is equal to the ground truth label $y_i^t$. Prompts with the highest $p_2$ reliability scores are assigned to
$\mathcal{S}_{\mathrm{unsup}}^t$, and the rest are passed to Stage-II. In Stage-II, CARE estimates the supervision value only within this predicted unreliable region. Given a wrong majority-vote pseudo-label, the CAG is determined by the known majority cluster and the number of correct non-majority responses. Thus, instead
of regressing CAG directly or predicting the ground-truth answer identity, $f_{\psi_2^t}^{\mathrm{cag}}$ predicts a distribution over this latent count:
\begin{equation}
\mathbf{d}_i^t=f_{\psi_2^t}^{\mathrm{cag}}(\phi_i^t),\qquad
\hat{s}_i^t=\sum_{k\in\mathcal{K}_i^t} d_{i,k}^t\,s_i^t(k),
\end{equation}
where $d_{i,k}^t$ is the predicted probability that exactly $k$ non-majority
responses are correct, $\mathcal{K}_i^t$ is the admissible set of such counts, and
$s_i^t(k)$ is the corresponding class-wise CAG. Among the Stage-II candidates,
prompts with the highest expected CAG scores under the annotation budget $p$ are
assigned to $\mathcal{S}_{\mathrm{sup}}^t$, while the rest are assigned to
$\mathcal{S}_{\mathrm{drop}}^t$.

\subsubsection{Online Training Procedure}
\label{sec:framework_training}

After partitioning, CARE updates the policy using only
$\mathcal{S}_{\mathrm{sup}}^t\cup\mathcal{S}_{\mathrm{unsup}}^t$; samples in
$\mathcal{S}_{\mathrm{drop}}^t$ are excluded. For annotated samples, rewards and
advantages are computed with the acquired ground-truth labels. For reliable
pseudo-labeled samples, CARE downweights the pseudo-label advantage by the Stage-I
reliability score:
\begin{equation}
A_{i,g}^{\mathrm{mix},t}
=
\begin{cases}
A_{i,g}^{\star,t}, 
& i\in\mathcal{S}_{\mathrm{sup}}^t \; (\text{ground-truth label } y_i^t), \\[2pt]
c_i^{(1),t}\tilde{A}_{i,g}^t, 
& i\in\mathcal{S}_{\mathrm{unsup}}^t \; (\text{pseudo-label } \hat y_i^t).
\end{cases}
\label{eq:mix_a}
\end{equation}
CARE then optimizes Eq.~\eqref{eq:grpo_obj} using
$A_{i,g}^{\mathrm{mix},t}$. After the policy update, newly annotated samples are
added to the labeled buffer and provide online supervision for both classifiers,
yielding the next-round parameters $\psi_1^{t+1}$ and $\psi_2^{t+1}$.

\subsection{Detail Implementations}
\label{sec:detail_implement}

\paragraph{Cascaded Classifier Representation.}
The cascaded classifiers use lightweight rollout-derived representations that can be obtained with negligible additional cost. Following prior observations that hidden states, response length, and answer consistency are correlated with response correctness~\cite{zhang2025reasoning,azaria2023internal,su2025between,manakul2023selfcheckgpt}, we construct $\phi^t_i=E(x_i^t,\{o^t_{i,g}\}_{g=1}^{G})$ from two types of signals: prompt-level statistics summarizing the rollout set and response-level statistics describing individual responses. Specifically, we use the final-layer hidden states of the last token, the valid rollout ratio, the distribution of answer clusters, and normalized response lengths. These signals are directly available from standard rollout statistics or forward-pass hidden states, and therefore introduce negligible additional overhead. Detailed feature definitions are provided in Appendix~\ref{appendix:implementation_details}.

\paragraph{Cascaded Classifiers Network Design.}
To keep the acquisition classifiers lightweight, we implement both stages with MLPs and train them with cross-entropy loss. Besides, to improve classifier learning, we additionally use an auxiliary response-level head to predict the correctness of each response:
\begin{equation}
\mathcal{L}_{\mathrm{stage\text{-}1}}
=
\mathcal{L}_{\mathrm{classifier}}
+
\lambda_{\mathrm{aux}}\mathcal{L}_{\mathrm{aux}},
\label{eq:stage1_classifier_loss}
\end{equation}
where \(\mathcal{L}_{\mathrm{classifier}}\) is the main classification loss and \(\mathcal{L}_{\mathrm{aux}}\) is the auxiliary response-level correctness loss. We further use a training buffer and class-balanced sampling to improve sample efficiency and robustness. More details are provided in Appendix~\ref{appendix:implementation_details}.
\section{Theoretical Analysis}
\label{sec:theory}
In this section, we theoretically analyze the proposed CAG metric by connecting it to the deviation of the optimization direction. For a prompt $x$ and its sampled responses $\{o_i\}_{i=1}^G$, let $g_x$ denote the GRPO gradient computed with ground-truth rewards in \eqref{eq:grpo_obj}, and let $\hat g_x$ denote the gradient computed with pseudo rewards. We measure their directional alignment by
$\cos\theta_x=\langle g_x,\hat g_x\rangle/(\|g_x\|_2\|\hat g_x\|_2)$. For simplicity, we consider the strict on-policy GRPO setting, where the gradient is determined by the group-normalized advantage vector. Specifically, for an advantage vector $A=(A_1,\ldots,A_G)$, the corresponding gradient can be written as
$g(A)=\frac{1}{G}\sum_{i=1}^G A_i s_i$, where
$s_i=\nabla_\theta \sum_{j=1}^{|o_{i}|} \frac{1}{|o_i|}\pi_\theta(o_{i,j}\mid x,o_{i,<j})$ is the score vector of response $o_i$. We first show that the gradient cosine is equivalent to the cosine similarity between the transformed advantage vectors.

\begin{lemma}[Gradient cosine as a transformed advantage cosine] \label{lemma:transformed_cosine_similarity}
Let $S=[s_1,\ldots,s_G]$ be the score matrix and define $K_x=S^\top S$. Let $A$ and $\hat A$ be the advantage vectors induced by ground-truth and pseudo rewards, respectively. If $K_x \succ 0$, then
\begin{equation}
\cos\theta_x
=\frac{\langle g_x,\hat g_x\rangle}{\|g_x\|_2\|\hat g_x\|_2}
=\frac{\left\langle K_x^{1/2}A, K_x^{1/2}\hat A \right\rangle}{\|K_x^{1/2}A\|_2\|K_x^{1/2}\hat A\|_2}
\end{equation}
\end{lemma}

\textbf{Remark.} (Proof deferred to \appendixref{appendix:proof_of_lemma}.) The lemma shows that gradient alignment can be viewed as the cosine similarity between the transformed advantage vectors $K_x^{1/2}A$ and $K_x^{1/2}\hat A$. The matrix $K_x=S^\top S$ is the empirical NTK matrix over sampled responses, reflecting the intrinsic similarity between responses. The positive definiteness assumption of $K_x \succ 0$ is mild since $G$ is much smaller than the parameter dimension and therefore exact linear dependence among score vectors is unlikely. Hence, the deviation between $A$ and $\hat A$ directly affects the induced optimization direction.

The next theorem gives a direct connection between the proposed CAG metric, i.e., the advantage deviation, and the alignment of the induced optimization direction.

\begin{theorem}[CAG upper-bounds gradient alignment] \label{theorem:CAG_upperbound_gradient_alignment}
Let
$\mathcal H=\{v\in\mathbb R^G:\mathbf 1^\top v=0\}$
be the centered subspace. Assume $K_x \succ 0$ and define $\lambda_{\max}(K_x|_{\mathcal H})=\max_{\substack{v\in\mathcal H\\v\ne0}}\frac{v^\top K_xv}{v^\top v},\lambda_{\min}(K_x|_{\mathcal H})=\min_{\substack{v\in\mathcal H\\v\ne0}}\frac{v^\top K_xv}{v^\top v},$ and $\kappa(x)=\frac{\lambda_{\max}(K_x|_{\mathcal H})}{\lambda_{\min}(K_x|_{\mathcal H})}$. Then, for any nonzero CAG score $d_x=\|A-\hat A\|_2>0$, the gradient alignment satisfies
\begin{equation}
\cos\theta_x \le 1- \frac{2}{4\kappa(x)\frac{G}{d_x^2}-(\kappa(x)-1)}.
\end{equation}
\end{theorem}

\textbf{Remark.} (Proof deferred to~\appendixref{appendix:proof_of_theorem}.) The theorem provides theoretical insight for CAG being a meaningful indicator of pseudo-label quality: a larger advantage deviation $d_x$ yields a smaller guaranteed upper bound on the gradient alignment, implying that pseudo labels may introduce a more misaligned optimization signal. The bound also shows that this relationship depends on the geometry of the sampled responses through the condition number $\kappa(x)$. When $\kappa(x)$ is large, the same advantage deviation can have a less predictable effect on the gradient direction due to anisotropic score-vector geometry. Notably, when $\kappa(x)$ takes the minimum value, i.e., $\kappa(x)=1$, the bound becomes tight and reduces to the exact identity $\cos\theta_x=1-\frac{d_x^2}{2G}$.

\section{Experiment}
\subsection{Experiment Setup}
\paragraph{Baselines.}
We compare CARE against four groups of baselines.
(1) \textbf{Unsupervised training.} We use TTRL as a representative method.
(2) \textbf{Semi-supervised RLVR methods.} Since our setting differs from semi-RLVR, we adapt semi-RLVR~\cite{yang2025trapo,luo2026memreward} to the RLAVR setting by randomly selecting samples for ground-truth annotation.
(3) \textbf{Active learning methods.} We consider three selection criteria: Entropy over the answer-cluster distribution, Prob based on the average response probability, and Oracle based on the ground-truth CAG score. 
For both semi-supervised and active-learning baselines, the selected samples receive ground-truth annotations, while the remaining samples are assigned majority-voting labels.
(4) \textbf{Fully supervised training.} GT uses ground-truth labels for all samples.
All methods are trained with GRPO~\cite{shao2024deepseekmath}.

\paragraph{Implementation and Evaluation.}
Following URLVR~\cite{he2026far}, we construct the training set from DAPO-17k~\cite{yu2025dapo} and train all methods for two epochs using Qwen3-Base~\cite{yang2025qwen3} backbones (1.7B, 4B, and 8B) and Phi4-mini-instruct~\cite{abdin2024phi}. We set the labeling budget to $p=20\%$. Evaluation is performed on out-of-domain mathematical reasoning benchmarks: MATH500~\cite{hendrycks2021measuring}, AIME24, AIME25, AMC23, Hmmt Feb 25, and Olympiad~\cite{he2024olympiadbench}. Following URLVR, we set the sampling temperature to 0.6 and use average@8 for evaluation, applying average@32 for smaller datasets (AIME, AMC, Hmmt) to improve stability. Correctness is determined by comparing each answer with the corresponding reference. Further details on the datasets and implementation are provided in Appendices~\ref{appendix:dataset_details} and~\ref{appendix:implementation_details}. Additional results on Phi4-mini-instruct are reported in Appendix~\ref{appendix:phi4_main_exp}. We also conduct evaluations in the logical reasoning domain using the K\&K dataset~\cite{xie2025memorization}, as shown in Appendix~\ref{appendix:kk_main_exp}.

\subsection{Main Result}
\textbf{CARE improves training effectiveness and stabilizes training.} As shown in Table~\ref{tab:main-results}, CARE consistently outperforms the unsupervised training method TTRL, semi-supervised training method Random and other active-learning-based label selection methods, including  Entropy, Prob and Oracle. Figure~\ref{fig:main_exp} (left) shows that most methods suffer from training collapse under erroneous labels. Oracle largely alleviates this issue by selecting for annotation the samples that are most misleading to the training signal, validating the rationality of the CAG metric; however, computing CAG relies on ground-truth labels. CARE instead predicts CAG with classifiers and selects high-confidence samples for unsupervised training, which not only stabilizes training but also achieves better performance, while maintaining a relatively small gap to GT, demonstrating the effectiveness of CARE.

\textbf{Proper use of unsupervised samples has the potential to improve training.} To disentangle the effect of training collapse on active learning methods, we further evaluate Random, Prob, Entropy, and Oracle without their unsupervised samples, training them only on the annotated samples selected by each method. As shown in Figure~\ref{fig:main_exp} (right), these variants consistently outperform their counterparts trained with additional unsupervised samples, suggesting that excessive erroneous pseudo-labeled samples may even hinder training. In contrast, CARE achieves better performance than most baselines by leveraging high-confidence unsupervised samples, demonstrating both the value of properly utilizing unsupervised data and the effectiveness of CARE.

\textbf{CARE incurs little computational overhead.} To assess computational overhead, we compare the per-step training time of different active learning methods using Qwen3-4B-Base. As shown in Figure~\ref{fig:analysis_exp} (a), CARE incurs rollout and policy optimization costs comparable to other methods, since classifier features are extracted during rollout and the classifier itself is lightweight. Moreover, dropping samples filtered by the second-stage classifier reduces overhead by 21.03\%, demonstrating CARE's efficiency. Appendix~\ref{appendix:train_time} reports the total training time, where CARE maintains a relatively low overall cost even when sample dropping is disabled for fair comparison.

\begin{figure}[t]
    \centering
    \begin{minipage}[t]{0.54\linewidth}
        \centering
        \includegraphics[width=\linewidth]{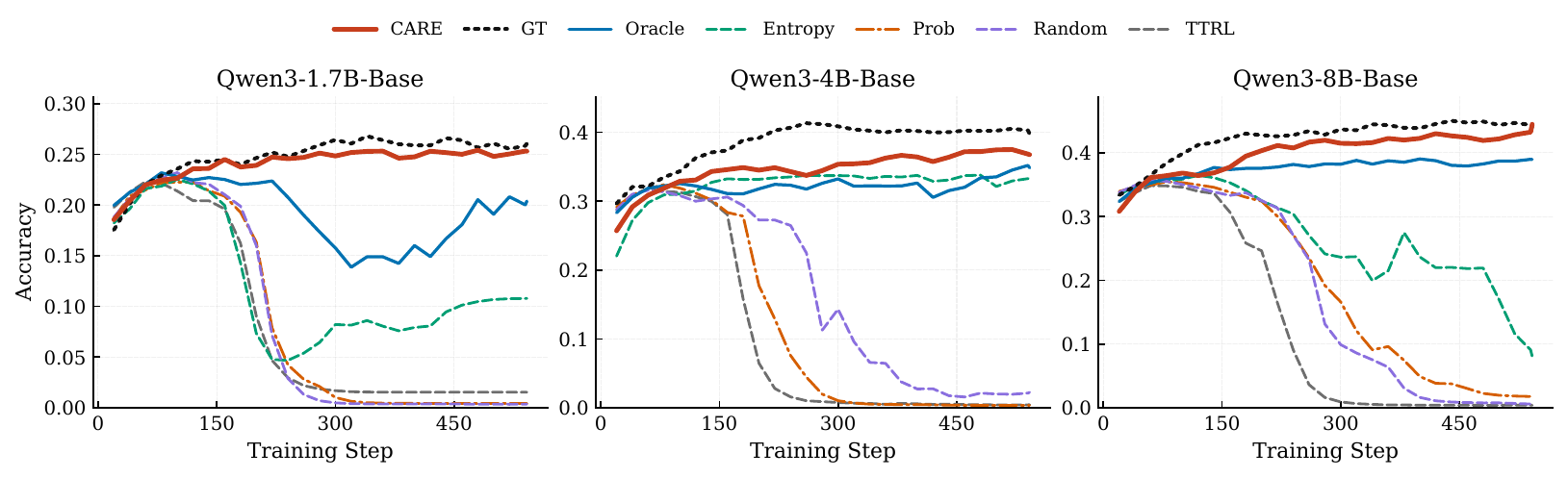}
    \end{minipage}
    \hfill
    \begin{minipage}[t]{0.44\linewidth}
        \centering
        \includegraphics[width=\linewidth]{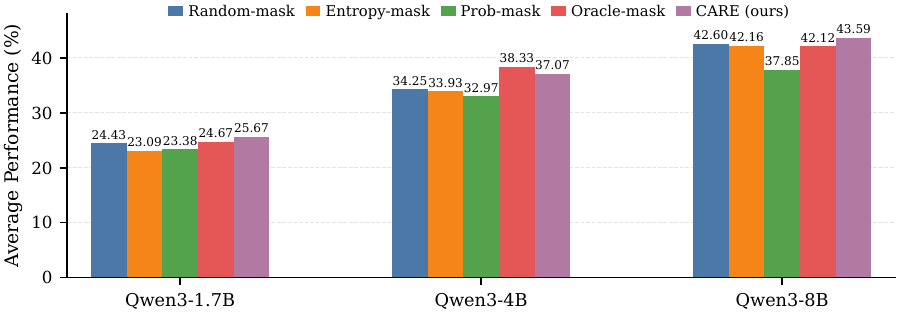}
    \end{minipage}
    \caption{\textbf{(Left)} Comparison of training curves among baseline methods on math tasks. \textbf{(Right)} Comparison of baseline training performance on math tasks after removing unsupervised samples.}
    \label{fig:main_exp}
\end{figure}

\begin{table}[t]
\centering
\small
\setlength{\tabcolsep}{3pt} 
\renewcommand{\arraystretch}{1.2}
\caption{Performance (\%) of Qwen3-Base models on math tasks. Vanilla and GT are reported as reference settings and are excluded from best/second-best ranking. The \textbf{bold} and \underline{underline} indicate the best and second-best results among the remaining methods, respectively.}

\resizebox{0.7\linewidth}{!}{%
\begin{tabular}{lccccccc}
\toprule
\textbf{Methods} & \textbf{AIME24} & \textbf{AIME25} & \textbf{MATH500} & \textbf{AMC23} & \textbf{Hmmt25} & \textbf{Olympiad} & \textbf{Average} \\
\hline
\rowcolor{mygray}
\multicolumn{8}{l}{\textbf{\emph{Qwen3-1.7B-Base}}} \\
\rowcolor{gray!8}
Vanilla & 0.25 & 0.24 & 40.45 & 25.62 & 0.00 & 18.38 & 14.89 \\
\rowcolor{gray!8}
GT  & 8.02 & 4.69 & 67.65 & 47.73 & 1.35 & 30.70 & 26.69 \\
TTRL & \underline{6.88} & 4.17 & \underline{63.65} & \underline{36.33} & 0.10 & \underline{27.99} & \underline{23.18} \\
Random & 6.15 & \underline{5.10} & 62.83 & 35.70 & \underline{0.31} & 27.08 & 22.86 \\
Entropy & 5.52 & 4.79 & 62.02 & 35.47 & 0.10 & 27.28 & 22.53 \\
Prob & 5.94 & 3.44 & 62.73 & 33.98 & 0.10 & 27.19 & 22.23 \\
Oracle & 5.73 & 4.58 & 62.88 & 36.02 & 0.10 & 27.27 & 22.76 \\
\rowcolor{yellow!20}
\textbf{CARE (ours)} & \textbf{9.27} & \textbf{5.83} & \textbf{67.50} & \textbf{40.08} & \textbf{0.73} & \textbf{30.62} & \textbf{25.67} \\
\hline

\rowcolor{mygray}
\multicolumn{8}{l}{\textbf{\emph{Qwen3-4B-Base}}} \\
\rowcolor{gray!8}
Vanilla & 8.02 & 3.44 & 31.17 & 22.73 & 0.52 & 22.46 & 14.72 \\
\rowcolor{gray!8}
GT  & 22.18 & 20.62 & 83.78 & 60.62 & 8.44 & 51.17 & 41.14 \\
TTRL & 11.35 & 7.92 & 76.23 & 53.52 & 1.15 & 42.01 & 32.03 \\
Random & 10.42 & 8.12 & 76.20 & 55.23 & 0.42 & 42.54 & 32.15 \\
Entropy & 11.67 & 7.71 & 76.42 & 51.09 & 1.15 & 41.20 & 31.54 \\
Prob & 11.98 & \underline{10.42} & 76.48 & 54.45 & 0.83 & 42.29 & 32.74 \\
Oracle & \underline{12.60} & 9.05 & \underline{79.60} & \underline{60.08} & \underline{3.02} & \underline{44.95} & \underline{34.95} \\
\rowcolor{yellow!20}
\textbf{CARE (ours)} & \textbf{14.79} & \textbf{16.04} & \textbf{80.18} & \textbf{61.02} & \textbf{4.27} & \textbf{46.12} & \textbf{37.07} \\
\hline

\rowcolor{mygray}
\multicolumn{8}{l}{\textbf{\emph{Qwen3-8B-Base}}} \\
\rowcolor{gray!8}
Vanilla & 10.83 & 7.29 & 55.40 & 43.13 & 1.46 & 34.23 & 25.39 \\
\rowcolor{gray!8}
GT  & 24.79 & 19.06 & 87.08 & 70.31 & 10.63 & 55.70 & 44.59 \\
TTRL & 16.15 & 13.96 & 80.45 & 60.16 & 2.92 & 45.47 & 36.52 \\
Random & 14.58 & 11.67 & 79.70 & 58.91 & 2.19 & 45.71 & 35.46 \\
Entropy & 16.15 & 13.02 & 80.67 & 61.09 & 4.06 & 46.28 & 36.88 \\
Prob & 15.52 & 12.92 & 80.12 & 60.47 & 1.98 & 45.91 & 36.15 \\
Oracle & \underline{18.96} & \underline{15.83} & \underline{82.17} & \textbf{66.41} & \underline{5.62} & \underline{48.06} & \underline{39.51} \\
\rowcolor{yellow!20}
\textbf{CARE (ours)} & \textbf{24.79} & \textbf{20.83} & \textbf{85.60} & \underline{65.23} & \textbf{10.73} & \textbf{54.34} & \textbf{43.59} \\
\bottomrule
\end{tabular}%
}

\label{tab:main-results}
\end{table}

\subsection{Ablation Study}
\textbf{Validation of component design}. To assess each CARE component, we conduct ablation studies, with 4B results shown in Figure~\ref{fig:analysis_exp}(b) and others in Appendix~\ref{appendix:more_ablation_results}. We compare classification-based CAG prediction with regression-based (\textit{CARE-mse}), evaluate the first-stage classifier by removing pseudo-labeled samples (\textit{w/o ps}), and assess the second-stage classifier by randomly selecting $p\%$ of remaining samples (\textit{w/o stage2}). Auxiliary and class-balancing losses are ablated individually (\textit{w/o aux}, \textit{w/o cb}). \textit{CARE-mse} yields inferior CAG predictions and weaker downstream performance, demonstrating the advantage of classification-based CAG prediction. \textit{w/o ps} reduces performance across all models, confirming the value of pseudo-labeled samples for low-risk unsupervised training. \textit{w/o stage2} also degrades performance, highlighting the necessity of the second-stage classifier. Finally, \textit{w/o aux} and \textit{w/o cb} show minor variations, but using both auxiliary and class-balancing losses improves average performance, supporting the design of each CARE component.

\textbf{Analysis of the Unsupervised Sample Selection Ratio $p_2$}. We analyze the effect of different $p_2$ values on CARE in Figure~\ref{fig:analysis_exp}(c), with additional results in Appendix~\ref{app:p2_analysis}. For Qwen3-4B-Base, classifier accuracy remains relatively stable across $p_2$ values. Increasing $p_2$ introduces more training samples and further improves performance, with $p_2=0.25$ achieving the best result.

\paragraph{Analysis of Annotation Budget Ratio $p$.} We compare CARE with the three strongest baselines on Qwen3-1.7B-Base model under varying annotation budget ratios $p$ in Figure~\ref{fig:analysis_exp}(d). CARE consistently performs best across all budgets. As $p$ increases, performance improves: CARE matches GT with only $40\%$ of annotations and surpasses GT at $80\%$. These results show that CARE remains effective under limited annotation budgets and maintains consistent advantages over baselines.

\begin{figure}[t]
    \centering
    \begin{minipage}[t]{0.20\linewidth}
        \centering
        \includegraphics[width=\linewidth]{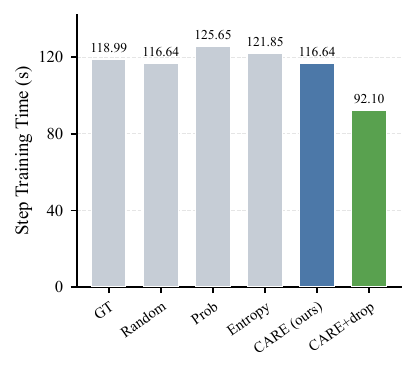}
    \end{minipage}
    \hfill
    \begin{minipage}[t]{0.26\linewidth}
        \centering
        \includegraphics[width=\linewidth]{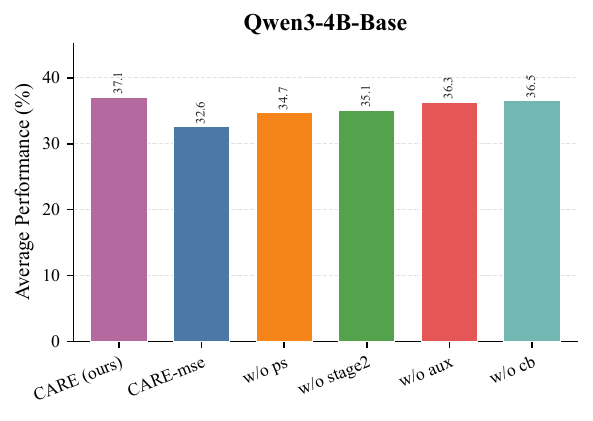}
    \end{minipage}
    \hfill
    \begin{minipage}[t]{0.24\linewidth}
        \centering
        \includegraphics[width=\linewidth]{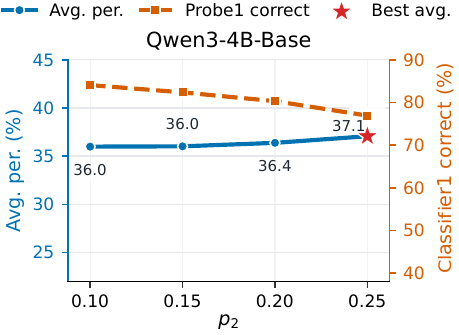}
    \end{minipage}
    \hfill
    \begin{minipage}[t]{0.25\linewidth}
        \centering
        \includegraphics[width=\linewidth]{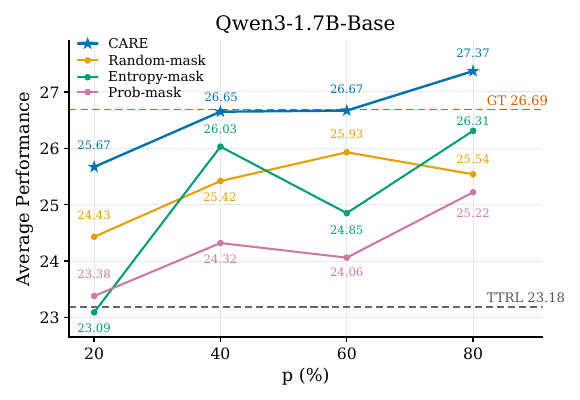}
    \end{minipage}

    \caption{
    \textbf{(a)} Comparison of step training time across different baselines on math tasks.
    \textbf{(b)} Ablation study of CARE on math tasks with Qwen3-4B-Base.
    \textbf{(c)} Hyperparameter analysis of $p_2$ on math tasks with Qwen3-4B-Base.
    \textbf{(d)} Hyperparameter analysis of $p$ on math tasks with Qwen3-1.7B-Base.
    }
    \label{fig:analysis_exp}
\end{figure}
\section{Conclusion}
In this paper, we propose RLAVR, a new setting that actively acquires ground-truth annotations for a subset of samples under a limited annotation budget and integrates the verified labels with pseudo-labels for training, thereby achieving competitive performance with high sample efficiency. We investigate the underlying causes of training collapse in current unsupervised RLVR pipelines and introduce a novel metric, CAG, to quantify the annotation value of individual samples, further substantiating its validity through theoretical analysis. Building upon these insights, we propose CARE, a framework that leverages a cascaded classifier network to predict CAG, perform label selection, and conduct RLAVR training. This mechanism simultaneously stabilizes the training dynamics and enhances overall performance. Extensive empirical evaluations across diverse domains, model scales, and model families unequivocally demonstrate the efficacy of CARE.


\bibliographystyle{plain}
\bibliography{refs}\

@article{shao2024deepseekmath,
  title={Deepseekmath: Pushing the limits of mathematical reasoning in open language models},
  author={Shao, Zhihong and Wang, Peiyi and Zhu, Qihao and Xu, Runxin and Song, Junxiao and Bi, Xiao and Zhang, Haowei and Zhang, Mingchuan and Li, YK and Wu, Yang and others},
  journal={arXiv preprint arXiv:2402.03300},
  year={2024}
}

@article{wang2025beyond,
  title={Beyond Majority Voting: Towards Fine-grained and More Reliable Reward Signal for Test-Time Reinforcement Learning},
  author={Wang, Weiqin and Wang, Yile and Chen, Kehao and Huang, Hui},
  journal={arXiv preprint arXiv:2512.15146},
  year={2025}
}

@article{guo2025deepseek,
  title={Deepseek-r1: Incentivizing reasoning capability in llms via reinforcement learning},
  author={Guo, Daya and Yang, Dejian and Zhang, Haowei and Song, Junxiao and Wang, Peiyi and Zhu, Qihao and Xu, Runxin and Zhang, Ruoyu and Ma, Shirong and Bi, Xiao and others},
  journal={arXiv preprint arXiv:2501.12948},
  year={2025}
}

@article{zuo2025ttrl,
  title={Ttrl: Test-time reinforcement learning},
  author={Zuo, Yuxin and Zhang, Kaiyan and Sheng, Li and Qu, Shang and Cui, Ganqu and Zhu, Xuekai and Li, Haozhan and Zhang, Yuchen and Long, Xinwei and Hua, Ermo and others},
  journal={arXiv preprint arXiv:2504.16084},
  year={2025}
}

@article{zhang2025right,
  title={Right question is already half the answer: Fully unsupervised llm reasoning incentivization},
  author={Zhang, Qingyang and Wu, Haitao and Zhang, Changqing and Zhao, Peilin and Bian, Yatao},
  journal={arXiv preprint arXiv:2504.05812},
  year={2025}
}

@article{zhao2025learning,
  title={Learning to reason without external rewards},
  author={Zhao, Xuandong and Kang, Zhewei and Feng, Aosong and Levine, Sergey and Song, Dawn},
  journal={arXiv preprint arXiv:2505.19590},
  year={2025}
}

@article{agarwal2025unreasonable,
  title={The unreasonable effectiveness of entropy minimization in llm reasoning},
  author={Agarwal, Shivam and Zhang, Zimin and Yuan, Lifan and Han, Jiawei and Peng, Hao},
  journal={arXiv preprint arXiv:2505.15134},
  year={2025}
}

@article{prabhudesai2025maximizing,
  title={Maximizing confidence alone improves reasoning},
  author={Prabhudesai, Mihir and Chen, Lili and Ippoliti, Alex and Fragkiadaki, Katerina and Liu, Hao and Pathak, Deepak},
  journal={arXiv preprint arXiv:2505.22660},
  year={2025}
}

@article{yu2025restrain,
  title={RESTRAIN: From Spurious Votes to Signals--Self-Driven RL with Self-Penalization},
  author={Yu, Zhaoning and Su, Will and Tao, Leitian and Wang, Haozhu and Singh, Aashu and Yu, Hanchao and Wang, Jianyu and Gao, Hongyang and Yuan, Weizhe and Weston, Jason and others},
  journal={arXiv preprint arXiv:2510.02172},
  year={2025}
}

@article{zhang2025consistent,
  title={Consistent paths lead to truth: Self-rewarding reinforcement learning for llm reasoning},
  author={Zhang, Kongcheng and Yao, Qi and Liu, Shunyu and Wang, Yingjie and Lai, Baisheng and Ye, Jieping and Song, Mingli and Tao, Dacheng},
  journal={arXiv preprint arXiv:2506.08745},
  year={2025}
}

@article{ghimire2026prism,
  title={PRISM: A Unified Framework for Post-Training LLMs Without Verifiable Rewards},
  author={Ghimire, Mukesh and Feng, Aosong and You, Liwen and Luo, Youzhi and Liu, Fang and Zhu, Xuan},
  journal={arXiv preprint arXiv:2601.04700},
  year={2026}
}

@article{lu2026contextual,
  title={Contextual Rollout Bandits for Reinforcement Learning with Verifiable Rewards},
  author={Lu, Xiaodong and Wang, Xiaohan and Chai, Jiajun and Yin, Guojun and Lin, Wei and Chen, Zhijun and Luo, Yu and Zhuang, Fuzhen and Ban, Yikun and Wang, Deqing},
  journal={arXiv preprint arXiv:2602.08499},
  year={2026}
}

@article{lin2025awpo,
  title={AWPO: Enhancing Tool-Use of Large Language Models through Adaptive Integration of Reasoning Rewards},
  author={Lin, Zihan and Wang, Xiaohan and Yang, Hexiong and Chai, Jiajun and Cao, Jie and Yin, Guojun and Lin, Wei and He, Ran},
  journal={arXiv preprint arXiv:2512.19126},
  year={2025}
}

@article{lin2025rest,
  title={ResT: Reshaping Token-Level Policy Gradients for Tool-Use Large Language Models},
  author={Lin, Zihan and Wang, Xiaohan and Cao, Jie and Chai, Jiajun and Yin, Guojun and Lin, Wei and He, Ran},
  journal={arXiv preprint arXiv:2509.21826},
  year={2025}
}

@article{yang2025trapo,
  title={TraPO: A Semi-Supervised Reinforcement Learning Framework for Boosting LLM Reasoning},
  author={Yang, Shenzhi and Zhu, Guangcheng and Zheng, Xing and MA, Yingfan and Chen, Zhongqi and Song, Bowen and Wang, Weiqiang and Zhao, Junbo and Chen, Gang and Wang, Haobo},
  journal={arXiv preprint arXiv:2512.13106},
  year={2025}
}

@article{he2026far,
  title={How Far Can Unsupervised RLVR Scale LLM Training?},
  author={He, Bingxiang and Zuo, Yuxin and Liu, Zeyuan and Zhao, Shangziqi and Fu, Zixuan and Yang, Junlin and Qian, Cheng and Zhang, Kaiyan and Fan, Yuchen and Cui, Ganqu and others},
  journal={arXiv preprint arXiv:2603.08660},
  year={2026}
}

@article{tan2025diagnosing,
  title={Diagnosing and Mitigating System Bias in Self-Rewarding RL},
  author={Tan, Chuyi and Yuan, Peiwen and Wang, Xinglin and Li, Yiwei and Feng, Shaoxiong and Zhang, Yueqi and Shi, Jiayi and Zhang, Ji and Pan, Boyuan and Hu, Yao and others},
  journal={arXiv preprint arXiv:2510.08977},
  year={2025}
}

@article{zhou2025evolving,
  title={Evolving language models without labels: Majority drives selection, novelty promotes variation},
  author={Zhou, Yujun and Liang, Zhenwen and Liu, Haolin and Yu, Wenhao and Panaganti, Kishan and Song, Linfeng and Yu, Dian and Zhang, Xiangliang and Mi, Haitao and Yu, Dong},
  journal={arXiv preprint arXiv:2509.15194},
  year={2025}
}

@article{pan2026coverrl,
  title={CoVerRL: Breaking the Consensus Trap in Label-Free Reasoning via Generator-Verifier Co-Evolution},
  author={Pan, Teng and Yan, Yuchen and Wang, Zixuan and Zhang, Ruiqing and Han, Gaiyang and Zhang, Wanqi and Lu, Weiming and Xiao, Jun and Shen, Yongliang},
  journal={arXiv preprint arXiv:2603.17775},
  year={2026}
}

@article{luo2026memreward,
  title={MemReward: Graph-Based Experience Memory for LLM Reward Prediction with Limited Labels},
  author={Luo, Tianyang and Feng, Tao and Hua, Zhigang and Xie, Yan and Yang, Shuang and Liu, Ge and You, Jiaxuan},
  journal={arXiv preprint arXiv:2603.19310},
  year={2026}
}

@article{yi2026learn,
  title={Learn More with Less: Uncertainty Consistency Guided Query Selection for RLVR},
  author={Yi, Hao and Hu, Yulan and Li, Xin and Ouyang, Sheng and Ding, Lizhong and Liu, Yong},
  journal={arXiv preprint arXiv:2601.22595},
  year={2026}
}

@article{ji2024reinforcement,
  title={Reinforcement learning from human feedback with active queries},
  author={Ji, Kaixuan and He, Jiafan and Gu, Quanquan},
  journal={arXiv preprint arXiv:2402.09401},
  year={2024}
}

@inproceedings{feng2025duo,
  title={DUO: Diverse, Uncertain, On-Policy Query Generation and Selection for Reinforcement Learning from Human Feedback},
  author={Feng, Xuening and Jiang, Zhaohui and Kaufmann, Timo and Xu, Puchen and H{\"u}llermeier, Eyke and Weng, Paul and Zhu, Yifei},
  booktitle={Proceedings of the AAAI Conference on Artificial Intelligence},
  volume={39},
  number={16},
  pages={16604--16612},
  year={2025}
}

@article{duan2025efficient,
  title={Efficient process reward model training via active learning},
  author={Duan, Keyu and Liu, Zichen and Mao, Xin and Pang, Tianyu and Chen, Changyu and Chen, Qiguang and Shieh, Michael Qizhe and Dou, Longxu},
  journal={arXiv preprint arXiv:2504.10559},
  year={2025}
}

@incollection{aggarwal2014active,
  title={Active learning: A survey},
  author={Aggarwal, Charu C and Kong, Xiangnan and Gu, Quanquan and Han, Jiawei and Yu, Philip S},
  booktitle={Data classification},
  pages={599--634},
  year={2014},
  publisher={Chapman and Hall/CRC}
}

@article{bai2022training,
  title={Training a helpful and harmless assistant with reinforcement learning from human feedback},
  author={Bai, Yuntao and Jones, Andy and Ndousse, Kamal and Askell, Amanda and Chen, Anna and DasSarma, Nova and Drain, Dawn and Fort, Stanislav and Ganguli, Deep and Henighan, Tom and others},
  journal={arXiv preprint arXiv:2204.05862},
  year={2022}
}

@inproceedings{wang2014new,
  title={A new active labeling method for deep learning},
  author={Wang, Dan and Shang, Yi},
  booktitle={2014 International joint conference on neural networks (IJCNN)},
  pages={112--119},
  year={2014},
  organization={IEEE}
}

@article{geifman2017deep,
  title={Deep active learning over the long tail},
  author={Geifman, Yonatan and El-Yaniv, Ran},
  journal={arXiv preprint arXiv:1711.00941},
  year={2017}
}

@article{rafailov2023direct,
  title={Direct preference optimization: Your language model is secretly a reward model},
  author={Rafailov, Rafael and Sharma, Archit and Mitchell, Eric and Manning, Christopher D and Ermon, Stefano and Finn, Chelsea},
  journal={Advances in neural information processing systems},
  volume={36},
  pages={53728--53741},
  year={2023}
}

@article{liu2024dual,
  title={Dual active learning for reinforcement learning from human feedback},
  author={Liu, Pangpang and Shi, Chengchun and Sun, Will Wei},
  journal={arXiv preprint arXiv:2410.02504},
  year={2024}
}

@article{mehta2023sample,
  title={Sample efficient reinforcement learning from human feedback via active exploration},
  author={Mehta, Viraj and Das, Vikramjeet and Neopane, Ojash and Dai, Yijia and Bogunovic, Ilija and Schneider, Jeff and Neiswanger, Willie},
  year={2023}
}

@article{kee2018query,
  title={Query-by-committee improvement with diversity and density in batch active learning},
  author={Kee, Seho and Del Castillo, Enrique and Runger, George},
  journal={Information Sciences},
  volume={454},
  pages={401--418},
  year={2018},
  publisher={Elsevier}
}

@inproceedings{cai2013maximizing,
  title={Maximizing expected model change for active learning in regression},
  author={Cai, Wenbin and Zhang, Ya and Zhou, Jun},
  booktitle={2013 IEEE 13th international conference on data mining},
  pages={51--60},
  year={2013},
  organization={IEEE}
}

@article{mussmann2022active,
  title={Active learning with expected error reduction},
  author={Mussmann, Stephen and Reisler, Julia and Tsai, Daniel and Mousavi, Ehsan and O'Brien, Shayne and Goldszmidt, Moises},
  journal={arXiv preprint arXiv:2211.09283},
  year={2022}
}

@inproceedings{wang2021density,
  title={Density weighted diversity based query strategy for active learning},
  author={Wang, Tingting and Zhao, Xufeng and Lv, Qiujian and Hu, Bo and Sun, Degang},
  booktitle={2021 IEEE 24th International Conference on Computer Supported Cooperative Work in Design (CSCWD)},
  pages={156--161},
  year={2021},
  organization={IEEE}
}

@article{bayer2026activellm,
  title={Activellm: Large language model-based active learning for textual few-shot scenarios},
  author={Bayer, Markus and Lutz, Justin and Reuter, Christian},
  journal={Transactions of the Association for Computational Linguistics},
  volume={14},
  pages={1--22},
  year={2026},
  publisher={MIT Press 255 Main Street, 9th Floor, Cambridge, Massachusetts 02142, USA~…}
}

@inproceedings{kwon2023efficient,
  title={Efficient memory management for large language model serving with pagedattention},
  author={Kwon, Woosuk and Li, Zhuohan and Zhuang, Siyuan and Sheng, Ying and Zheng, Lianmin and Yu, Cody Hao and Gonzalez, Joseph and Zhang, Hao and Stoica, Ion},
  booktitle={Proceedings of the 29th symposium on operating systems principles},
  pages={611--626},
  year={2023}
}

@article{sheng2024hybridflow,
  title   = {HybridFlow: A Flexible and Efficient RLHF Framework},
  author  = {Guangming Sheng and Chi Zhang and Zilingfeng Ye and Xibin Wu and Wang Zhang and Ru Zhang and Yanghua Peng and Haibin Lin and Chuan Wu},
  year    = {2024},
  journal = {arXiv preprint arXiv: 2409.19256}
}

@article{yang2025qwen3,
  title={Qwen3 technical report},
  author={Yang, An and Li, Anfeng and Yang, Baosong and Zhang, Beichen and Hui, Binyuan and Zheng, Bo and Yu, Bowen and Gao, Chang and Huang, Chengen and Lv, Chenxu and others},
  journal={arXiv preprint arXiv:2505.09388},
  year={2025}
}

@article{yu2025dapo,
  title={Dapo: An open-source llm reinforcement learning system at scale},
  author={Yu, Qiying and Zhang, Zheng and Zhu, Ruofei and Yuan, Yufeng and Zuo, Xiaochen and Yue, Yu and Dai, Weinan and Fan, Tiantian and Liu, Gaohong and Liu, Lingjun and others},
  journal={arXiv preprint arXiv:2503.14476},
  year={2025}
}

@article{hendrycks2021measuring,
  title={Measuring mathematical problem solving with the math dataset},
  author={Hendrycks, Dan and Burns, Collin and Kadavath, Saurav and Arora, Akul and Basart, Steven and Tang, Eric and Song, Dawn and Steinhardt, Jacob},
  journal={arXiv preprint arXiv:2103.03874},
  year={2021}
}

@inproceedings{he2024olympiadbench,
  title={Olympiadbench: A challenging benchmark for promoting agi with olympiad-level bilingual multimodal scientific problems},
  author={He, Chaoqun and Luo, Renjie and Bai, Yuzhuo and Hu, Shengding and Thai, Zhen and Shen, Junhao and Hu, Jinyi and Han, Xu and Huang, Yujie and Zhang, Yuxiang and others},
  booktitle={Proceedings of the 62nd Annual Meeting of the Association for Computational Linguistics (Volume 1: Long Papers)},
  pages={3828--3850},
  year={2024}
}

@article{guan2025rstar,
  title={RStar-math: Small LLMs can master math reasoning with self-evolved deep thinking},
  author={Guan, Xinyu and Zhang, Li Lyna and Liu, Yifei and Shang, Ning and Sun, Youran and Zhu, Yi and Yang, Fan and Yang, Mao},
  journal={arXiv preprint arXiv:2501.04519},
  year={2025}
}

@article{wang2026survey,
  title={A survey on large language models for mathematical reasoning},
  author={Wang, Peng-Yuan and Liu, Tian-Shuo and Wang, Chenyang and Li, Ziniu and Wang, Yidi and Yan, Shu and Jia, Chengxing and Liu, Xu-Hui and Chen, Xinwei and Xu, Jiacheng and others},
  journal={ACM Computing Surveys},
  volume={58},
  number={8},
  pages={1--35},
  year={2026},
  publisher={ACM New York, NY}
}

@inproceedings{yang2025code,
  title={Code to think, think to code: A survey on code-enhanced reasoning and reasoning-driven code intelligence in llms},
  author={Yang, Dayu and Liu, Tianyang and Zhang, Daoan and Simoulin, Antoine and Liu, Xiaoyi and Cao, Yuwei and Teng, Zhaopu and Qian, Xin and Yang, Grey and Luo, Jiebo and others},
  booktitle={Proceedings of the 2025 Conference on Empirical Methods in Natural Language Processing},
  pages={2586--2616},
  year={2025}
}

@article{chen2025acereason,
  title={Acereason-nemotron: Advancing math and code reasoning through reinforcement learning},
  author={Chen, Yang and Yang, Zhuolin and Liu, Zihan and Lee, Chankyu and Xu, Peng and Shoeybi, Mohammad and Catanzaro, Bryan and Ping, Wei},
  journal={arXiv preprint arXiv:2505.16400},
  year={2025}
}

@article{zhang2025co,
  title={Co-rewarding: Stable Self-supervised RL for Eliciting Reasoning in Large Language Models},
  author={Zhang, Zizhuo and Zhu, Jianing and Ge, Xinmu and Zhao, Zihua and Zhou, Zhanke and Li, Xuan and Feng, Xiao and Yao, Jiangchao and Han, Bo},
  journal={arXiv preprint arXiv:2508.00410},
  year={2025}
}

@article{liao2026tool,
  title={Tool Verification for Test-Time Reinforcement Learning},
  author={Liao, Ruotong and R{\"o}hrich, Nikolai and Wang, Xiaohan and Zhang, Yuhui and Samadzadeh, Yasaman and Tresp, Volker and Yeung-Levy, Serena},
  journal={arXiv preprint arXiv:2603.02203},
  year={2026}
}

@article{zhang2025reasoning,
  title={Reasoning Models Know When They're Right: Probing Hidden States for Self-Verification},
  author={Zhang, Anqi and Chen, Yulin and Pan, Jane and Zhao, Chen and Panda, Aurojit and Li, Jinyang and He, He},
  journal={arXiv preprint arXiv:2504.05419},
  year={2025}
}

@inproceedings{azaria2023internal,
  title={The internal state of an LLM knows when it’s lying},
  author={Azaria, Amos and Mitchell, Tom},
  booktitle={Findings of the Association for Computational Linguistics: EMNLP 2023},
  pages={967--976},
  year={2023}
}

@inproceedings{manakul2023selfcheckgpt,
  title={Selfcheckgpt: Zero-resource black-box hallucination detection for generative large language models},
  author={Manakul, Potsawee and Liusie, Adian and Gales, Mark},
  booktitle={Proceedings of the 2023 conference on empirical methods in natural language processing},
  pages={9004--9017},
  year={2023}
}

@article{su2025between,
  title={Between underthinking and overthinking: An empirical study of reasoning length and correctness in llms},
  author={Su, Jinyan and Healey, Jennifer and Nakov, Preslav and Cardie, Claire},
  journal={arXiv preprint arXiv:2505.00127},
  year={2025}
}

@article{abdin2024phi,
  title={Phi-4 technical report},
  author={Abdin, Marah and Aneja, Jyoti and Behl, Harkirat and Bubeck, S{\'e}bastien and Eldan, Ronen and Gunasekar, Suriya and Harrison, Michael and Hewett, Russell J and Javaheripi, Mojan and Kauffmann, Piero and others},
  journal={arXiv preprint arXiv:2412.08905},
  year={2024}
}

@inproceedings{xie2025memorization,
  title={On memorization of large language models in logical reasoning},
  author={Xie, Chulin and Huang, Yangsibo and Zhang, Chiyuan and Yu, Da and Chen, Xinyun and Lin, Bill Yuchen and Li, Bo and Ghazi, Badih and Kumar, Ravi},
  booktitle={Proceedings of the 14th International Joint Conference on Natural Language Processing and the 4th Conference of the Asia-Pacific Chapter of the Association for Computational Linguistics},
  pages={2742--2785},
  year={2025}
}
\clearpage
\appendix
\section{Proofs}
\subsection{Proof \lemmaref{lemma:transformed_cosine_similarity}} \label{appendix:proof_of_lemma}
\begin{proof}
Under the strict on-policy setting, the gradients induced by the ground-truth and pseudo-reward advantages can be written as
\begin{equation}
g_x=\frac{1}{G}SA,
\qquad
\hat g_x=\frac{1}{G}S\hat A.
\end{equation}
Therefore,
\begin{equation}
\langle g_x,\hat g_x\rangle
=
\frac{1}{G^2}A^\top S^\top S\hat A
=
\frac{1}{G^2}A^\top K_x\hat A.
\end{equation}
Similarly, their norms satisfy
\begin{equation}
\|g_x\|_2
=
\frac{1}{G}\sqrt{A^\top K_xA},
\qquad
\|\hat g_x\|_2
=
\frac{1}{G}\sqrt{\hat A^\top K_x\hat A}.
\end{equation}
Substituting these identities into the definition of cosine similarity gives
\begin{equation}
\cos\theta_x
=
\frac{
A^\top K_x\hat A
}{
\sqrt{A^\top K_xA}\sqrt{\hat A^\top K_x\hat A}
}.
\end{equation}
Since $K_x\succ0$, its unique positive definite square root $K_x^{1/2}$ exists. Hence,
\begin{equation}
A^\top K_x\hat A
=
\left\langle K_x^{1/2}A,K_x^{1/2}\hat A\right\rangle,
\end{equation}
and
\begin{equation}
A^\top K_xA
=
\|K_x^{1/2}A\|_2^2,
\qquad
\hat A^\top K_x\hat A
=
\|K_x^{1/2}\hat A\|_2^2.
\end{equation}
Thus,
\begin{equation}
\cos\theta_x
=
\frac{\left\langle K_x^{1/2}A,K_x^{1/2}\hat A\right\rangle}
{\|K_x^{1/2}A\|_2\|K_x^{1/2}\hat A\|_2}.
\end{equation}
This completes the proof.
\end{proof}

\subsection{Proof \theoremref{theorem:CAG_upperbound_gradient_alignment}} \label{appendix:proof_of_theorem}
\begin{proof}
Since $A$ and $\hat A$ are group-normalized and centered, we have
$A,\hat A\in\mathcal H$ and
\begin{equation}
\|A\|_2^2=\|\hat A\|_2^2=G.
\end{equation}
Let
\begin{equation}
u=\frac{A}{\sqrt G},
\qquad
v=\frac{\hat A}{\sqrt G}.
\end{equation}
Then $u,v\in\mathcal H$ and $\|u\|_2=\|v\|_2=1$. Define
\begin{equation}
\gamma=u^\top v.
\end{equation}
Since $d_x=\|A-\hat A\|_2$, we have
\begin{equation}
d_x^2
=
\|A-\hat A\|_2^2
=
\|A\|_2^2+\|\hat A\|_2^2-2A^\top \hat A
=
2G(1-\gamma).
\end{equation}
Therefore,
\begin{equation}
\gamma
=
1-\frac{d_x^2}{2G}.
\end{equation}

By Lemma~\ref{lemma:transformed_cosine_similarity}, the gradient alignment is
\begin{equation}
\cos\theta_x
=
\frac{u^\top K_xv}
{\sqrt{u^\top K_xu}\sqrt{v^\top K_xv}}.
\end{equation}
Since multiplying $K_x$ by a positive scalar does not change the above cosine, we normalize $K_x$ on $\mathcal H$ so that its eigenvalues lie in $[1,\kappa(x)]$.

If $\gamma=-1$, then $v=-u$, hence $\cos\theta_x=-1$. In this case $d_x^2=4G$, and the claimed bound also equals $-1$. Therefore, it remains to consider $-1<\gamma<1$.

Let
\begin{equation}
c=\sqrt{\frac{1+\gamma}{2}},
\qquad
h=\sqrt{\frac{1-\gamma}{2}}.
\end{equation}
Define
\begin{equation}
p=\frac{u+v}{2c},
\qquad
q=\frac{u-v}{2h}.
\end{equation}
Then $p,q\in\mathcal H$ are orthonormal, and
\begin{equation}
u=cp+hq,
\qquad
v=cp-hq.
\end{equation}

It suffices to consider the two-dimensional subspace spanned by $p$ and $q$. Let
\begin{equation}
Q=[p,q]\in\mathbb R^{G\times 2},
\qquad
B=Q^\top K_x Q\in\mathbb R^{2\times 2}
\end{equation}
be the compression of $K_x$ onto this subspace. Since $B$ is symmetric positive definite, let its eigenvalues be $\mu\ge \nu>0$ and define
$\ell=\mu/\nu$. By the Rayleigh quotient characterization, we have
$1\le \ell\le \kappa(x)$. Since multiplying $B$ by a positive scalar does not change the cosine, we can normalize $\nu=1$. Therefore, up to a rotation, this compression can be written as
\begin{equation}
B=
R_\phi
\begin{pmatrix}
\ell & 0\\
0 & 1
\end{pmatrix}
R_\phi^\top,
\qquad
R_\phi=
\begin{pmatrix}
\cos\phi & -\sin\phi\\
\sin\phi & \cos\phi
\end{pmatrix}.
\end{equation}

Let $z=\cos 2\phi$. A direct calculation gives
\begin{equation}
F_\ell(z)
=
\frac{u^\top K_xv}{\sqrt{u^\top K_xu}\sqrt{v^\top K_xv}}
=
\frac{\gamma+\frac{\ell-1}{2}(\gamma+z)}{
\sqrt{1+(\ell-1)(1+\gamma z)+\frac{(\ell-1)^2}{4}(\gamma+z)^2}
}.
\end{equation}
Moreover,
\begin{equation}
\frac{\partial F_\ell(z)}{\partial z}
=
\frac{\ell(\ell-1)(1-\gamma^2)}{2\left[
1+(\ell-1)(1+\gamma z)
+\frac{(\ell-1)^2}{4}(\gamma+z)^2
\right]^{3/2}}
\ge 0.
\end{equation}
Thus, for fixed $\ell$, the maximum is attained at $z=1$. This gives
\begin{equation}
F_\ell(z)
\le
F_\ell(1)
=
\frac{
\ell(1+\gamma)-(1-\gamma)
}{
\ell(1+\gamma)+(1-\gamma)
}.
\end{equation}
Since
\begin{equation}
\frac{\partial}{\partial \ell}F_\ell(1)
=
\frac{
2(1-\gamma^2)
}{
\left[\ell(1+\gamma)+(1-\gamma)\right]^2
}
\ge 0,
\end{equation}
the maximum over all admissible $\ell$ is attained at $\ell=\kappa(x)$. Therefore,
\begin{equation}
\cos\theta_x
\le
\frac{
\kappa(x)(1+\gamma)-(1-\gamma)
}{
\kappa(x)(1+\gamma)+(1-\gamma)
}.
\end{equation}
Substituting $\gamma=1-d_x^2/(2G)$ yields
\begin{equation}
\cos\theta_x
\le
\frac{
4\kappa(x)G-(\kappa(x)+1)d_x^2
}{
4\kappa(x)G-(\kappa(x)-1)d_x^2
}.
\end{equation}
Equivalently,
\begin{equation}
\cos\theta_x
\le
1-
\frac{
2d_x^2
}{
4\kappa(x)G-(\kappa(x)-1)d_x^2
}
=
1-
\frac{
2
}{
4\kappa(x)\frac{G}{d_x^2}-(\kappa(x)-1)
}.
\end{equation}
This completes the proof.
\end{proof}

\section{Experimental Details}
\label{appendix:exp_details}
\subsection{Datasets for Main Experiments}
\label{appendix:dataset_details}
We summarize the number of samples in the training and evaluation datasets in Table~\ref{tab:datasets_statics}. 

\begin{table}[t]
\centering
\caption{Statistics of the training and test datasets used by all methods.}
\small
\setlength{\tabcolsep}{4pt}
\renewcommand{\arraystretch}{1.1}
\begin{tabular}{lll}
\toprule
\textbf{Dataset} & \textbf{Description
} & \textbf{\# Train / Test} \\
\midrule
DAPO-17k            & Math Reasoning     & 17398 / - \\
AIME24   & Math Reasoning     & - / 30 \\
AIME25   & Math Reasoning     & - / 30 \\
MATH500            & Math Reasoning     & - / 500 \\
AMC23            & Math Reasoning     & - / 40 \\
Hmmt25          & Math Reasoning     & - / 30 \\
Olympiad           & Math Reasoning     & - / 674 \\
K\&K train & Logic Reasoning     & 4500 / - \\
K\&K 3ppl & Logic Reasoning     & - / 100 \\
K\&K 4ppl & Logic Reasoning     & - / 100 \\
K\&K 5ppl & Logic Reasoning     & - / 100 \\
K\&K 6ppl & Logic Reasoning     & - / 100 \\
K\&K 7ppl & Logic Reasoning     & - / 100 \\
K\&K 8ppl & Logic Reasoning     & - / 100 \\
\bottomrule
\end{tabular}
\label{tab:datasets_statics}
\end{table}

\subsection{Implementation Details}
\label{appendix:implementation_details}
\paragraph{Framework Implementation.} Our framework is implemented on top of VeRL~\cite{sheng2024hybridflow} and URLVR~\cite{he2026far}. All experiments were conducted on 8$\times$H20 GPUs with 141GB of memory, using Python 3.10 and PyTorch 2.6. All models are optimized with the AdamW optimizer (\(\beta_1 = 0.9\), \(\beta_2 = 0.95\), weight decay 0.01) and accelerated via vLLM~\cite{kwon2023efficient}. We use a batch size of 64, a mini-batch size of 32, a maximum prompt length of 1024 tokens, a maximum response length of 4000 tokens, a learning rate of ($1\times10^{-6}$), and 8 GRPO rollouts. We train for 2 epochs on the DAPO-17k dataset~\cite{yu2025dapo} and 3 epochs on the K\&K dataset~\cite{xie2025memorization}. The KL loss is not used. For Phi4-mini-instruct, we set both the batch size and mini-batch size to 64 to stabilize GRPO training and ensure on-policy updates. For K\&K tasks, we find that GRPO training is unstable without regularization; therefore, we use the KL loss with its coefficient set to 0.01.

\paragraph{CARE and Baseline Implementation Details.} For the Entropy method, we first cluster the answer distribution and then compute the response entropy, selecting samples with the highest entropy for ground-truth annotation. For the Prob method, we compute the average probability of the response tokens and select samples with the lowest probability for ground-truth annotation. As shown in Table~\ref{tab:classifier_representation}, the CARE classifiers are instantiated as a two-stage MLP architecture with separate encoders for prompt- and response-level features. Across all experiments, we set the prompt encoder hidden dimension to 128, the prediction-head hidden dimension to 256, and the response encoder hidden/output dimensions to 64/512, with ReLU activations. The first-stage classifier includes a rollout-level auxiliary head with loss weight 1.5. The second-stage classifier outputs (G+1) logits for count prediction, where (G=8) is the number of rollouts per prompt; invalid count classes are masked before loss computation. We optimize the classifiers with AdamW using a learning rate of 1e-4. To stabilize online classifier learning, we maintain separate FIFO replay buffers for the two stages, each with a capacity of 2048, and mix 16 historical samples into each classifier update. For the second-stage classifier, historical samples are drawn using count-label stratification to mitigate class imbalance. In addition, we enable class-balanced losses, with class weights computed from batch-level class frequencies using power 0.5 and clipped to ([0.25,4.0]). For first-stage selection, the class-1 selection ratio $p_2$ for different models on math and K\&K tasks is shown in Table~\ref{tab:p2_settings}. In practice, this ratio can be adjusted based on the model capability and the dataset difficulty.

\begin{table}[t]
\centering
\caption{Classifier features used by the acquisition classifier.}
\label{tab:classifier_representation}
\small
\setlength{\tabcolsep}{4pt}
\begin{tabular}{p{0.25\linewidth} p{0.25\linewidth} p{0.42\linewidth}}
\toprule
\textbf{Representation} & \textbf{Feature} & \textbf{Definition} \\
\midrule
\multirow{3}{*}{
\begin{tabular}[c]{@{}l@{}}
Global representation \\
$\phi_i^{\mathrm{global}}$
\end{tabular}
}
&
Prompt representation
&
$h_i^{x}=\mathrm{Enc}(x_i)_{[-1]}\in\mathbb{R}^{d}$
\\
&
Valid-rollout ratio
&
$f_i^{\mathrm{valid}}=\frac{1}{n}\sum_{j=1}^{n}\mathbf{1}[o_{i,j}\neq \texttt{None}]$
\\
&
Cluster distribution
&
$f_i^{\mathrm{cluster}}=\mathrm{Sort}([p_{i,1},\ldots,p_{i,C}])$, where
$p_{i,c}=\frac{n_{i,c}}{\sum_{c'=1}^{C}n_{i,c'}}$
\\
\midrule
\multirow{4}{*}{
\begin{tabular}[c]{@{}l@{}}
Response-level \\
representation \\
$\Phi_i^{\mathrm{ans}}$
\end{tabular}
}
&
Response representation
&
$h_{i,g}^{\mathrm{ans}}=\mathrm{Enc}(x_i,o_{i,g})_{[-1]}\in\mathbb{R}^{d}$
\\
&
Normalized length
&
$\ell_{i,g}^{\mathrm{norm}}=|o_{i,g}|/L_{\max}$
\\
&
Response-level feature
&
$\phi_{i,g}^{\mathrm{ans}}=[h_{i,g}^{\mathrm{ans}};\ell_{i,g}^{\mathrm{norm}}]$
\\
&
Stacked response features
&
$\Phi_i^{\mathrm{ans}}=[\phi_{i,1}^{\mathrm{ans}};\ldots;\phi_{i,n}^{\mathrm{ans}}]$
\\
\bottomrule
\end{tabular}
\end{table}

\begin{table}[t]
\centering
\caption{Class-1 selection ratio $p_2$ used in the first-stage selection for different models and tasks.}
\label{tab:p2_settings}
\begin{tabular}{lcccc}
\toprule
Task & Qwen3-1.7B-Base & Qwen3-4B-Base & Qwen3-8B-Base & Phi4-mini-instruct \\
\midrule
Math & 0.10 & 0.25 & 0.25 & 0.25 \\
K\&K & 0.00 & 0.25 & 0.25 & 0.15 \\
\bottomrule
\end{tabular}
\end{table}

\section{More Experimental Results}
\subsection{Comparison of Computational Cost}
\label{appendix:train_time}
To compare the training overhead of different methods, we evaluate their total training cost using Qwen3-4B-Base as the backbone model on an 8$\times$H20 GPU setup with 141 GB of total memory, as summarized in Table~\ref{tab:time_cost}. For a fair comparison, during policy optimization, we do not discard the samples that CARE would otherwise exclude from training; instead, we set their advantages to zero. GT is trained on the full dataset and exhibits a substantial increase in average output length, resulting in a longer training time. Compared with the Random baseline, CARE incurs only an 8.25\% increase in training time while achieving substantial performance gains. This modest overhead stems from the lightweight design of its classifier network, which consists of only a small MLP. In addition, the total training time could be further reduced by excluding samples discarded after the second-stage classifier selection from policy optimization.

\begin{table}[t]
\caption{Total training time (hours) for various methods on math tasks with Qwen3-4B-Base.}
\label{tab:time_cost}
\centering
\small
\setlength{\tabcolsep}{4pt}
\renewcommand{\arraystretch}{1.1}
\begin{tabular}{lccc}
\toprule
\textbf{Method} & \textbf{Training} \\
\midrule
GT & 28.70 \\
Random-mask & 20.48 \\
Prob-mask & 20.43 \\
Entropy-mask & 22.13 \\
Oracle-mask & 21.51 \\
CARE (ours) & 22.17 \\
\bottomrule
\end{tabular}
\end{table}

\subsection{Additional Collapse Analysis Experiments}
\label{appendix:more_collapse_pre_exp}
For TTRL, we generate training labels from majority-voted responses. For the TTRL-mask variant, we compare the voted answers with the ground-truth labels and mask samples with incorrectly voted labels by assigning them a sample weight of zero before training. All other hyperparameters are kept the same as in the main experiments. In addition, we provide results on the Phi4-mini-instruct model in Figure~\ref{fig:phi_pre_exp} (Left), which exhibit trends consistent with those observed on the Qwen3 model. We also evaluate on K\&K tasks, with results for the Qwen3 and Phi4-mini-instruct models shown in Figure~\ref{fig:qwen_pre1_kk} and Figure~\ref{fig:phi_pre_exp_kk} (Left), respectively. The conclusions are consistent with those on math tasks.

\begin{figure}[t]
    \centering
    \begin{minipage}[t]{0.28\linewidth}
        \centering
        \includegraphics[width=\linewidth]{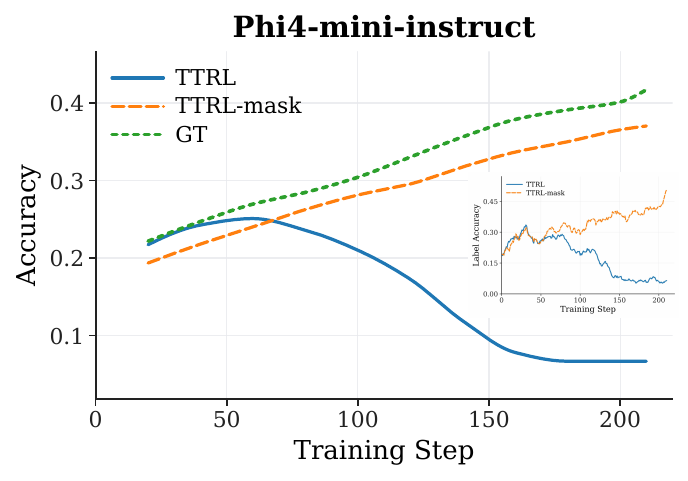}
    \end{minipage}
    \begin{minipage}[t]{0.28\linewidth}
        \centering
        \includegraphics[width=\linewidth]{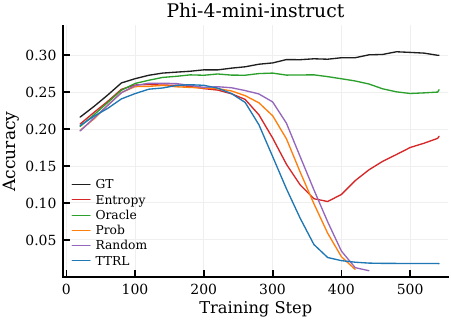}
    \end{minipage}
    \begin{minipage}[t]{0.28\linewidth}
        \centering
        \includegraphics[width=\linewidth]{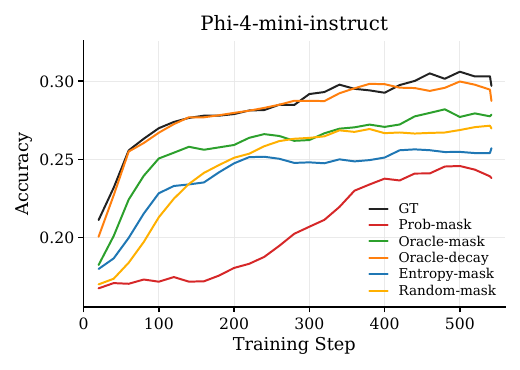}
    \end{minipage}
    \caption{\textbf{(Left)} Training dynamics of Phi4-mini-instruct on math tasks. Comparison of label acquisition strategies on math-task pseudo-labeled samples without masking \textbf{(Middle)} and with masking \textbf{(Right)}.}
    \label{fig:phi_pre_exp}
\end{figure}

\begin{figure}[t]
    \centering
    \begin{minipage}[t]{0.28\linewidth}
        \centering
        \includegraphics[width=\linewidth]{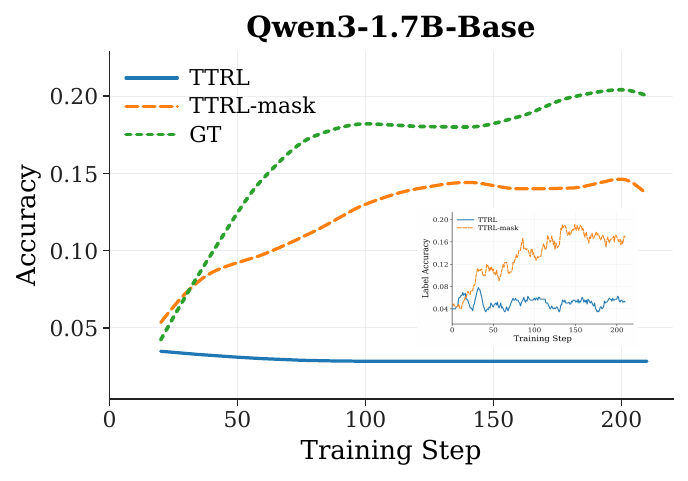}
    \end{minipage}
    \begin{minipage}[t]{0.28\linewidth}
        \centering
        \includegraphics[width=\linewidth]{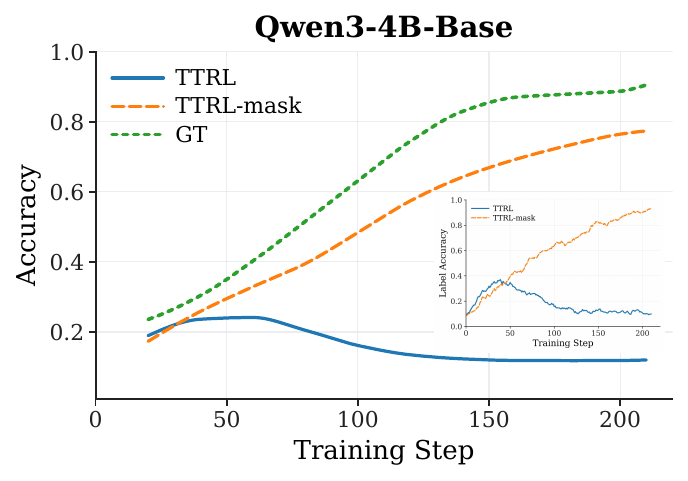}
    \end{minipage}
    \begin{minipage}[t]{0.28\linewidth}
        \centering
        \includegraphics[width=\linewidth]{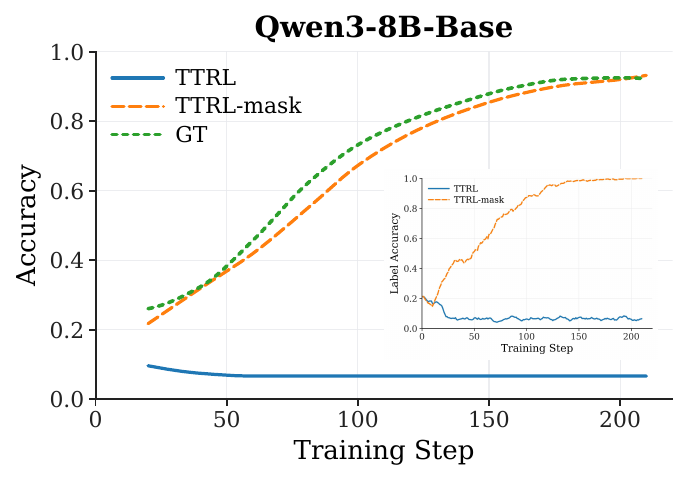}
    \end{minipage}
    \caption{Training dynamics across three Qwen3 models on K\&K tasks.}
    \label{fig:qwen_pre1_kk}
\end{figure}

\begin{figure}[t]
    \centering
    \begin{minipage}[t]{0.32\linewidth}
        \centering
        \includegraphics[width=\linewidth]{fig/phi_pre1_kk.pdf}
    \end{minipage}
    \begin{minipage}[t]{0.28\linewidth}
        \centering
        \includegraphics[width=\linewidth]{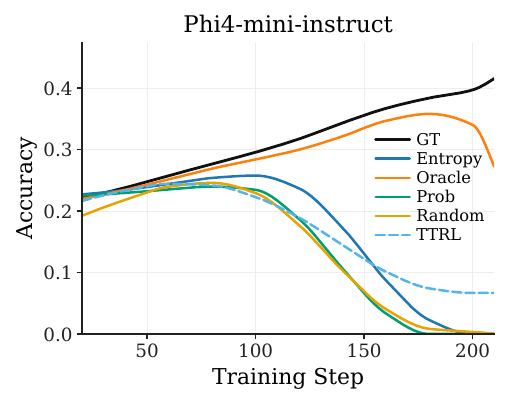}
    \end{minipage}
    \begin{minipage}[t]{0.28\linewidth}
        \centering
        \includegraphics[width=\linewidth]{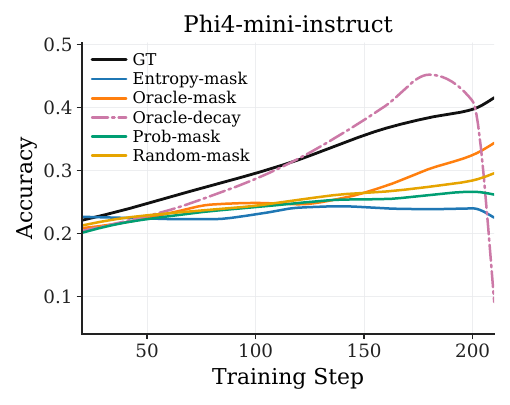}
    \end{minipage}
    \caption{\textbf{(Left)} Training dynamics of Phi4-mini-instruct on K\&K tasks. 
    Comparison of label acquisition strategies on K\&K pseudo-labeled samples without masking \textbf{(Middle)} and with masking \textbf{(Right)}.}
    \label{fig:phi_pre_exp_kk}
\end{figure}

\subsection{Additional Preliminary Experimental Results}
\label{appendix:pre_mask_exp}
Due to the definition of the CAG metric, correctly pseudo-labeled samples have a CAG value of zero. As a result, the Oracle strategy, which selects samples based on CAG, is more likely to acquire labels from incorrectly pseudo-labeled samples. To ensure a fair comparison, for the Random, Entropy, and Prob strategies, we also select labels only from incorrectly pseudo-labeled samples according to their respective selection criteria. The results on Phi4-mini-instruct for math and K\&K tasks without masking pseudo-labeled samples are shown in Figure~\ref{fig:phi_pre_exp} (Middle) and Figure~\ref{fig:phi_pre_exp_kk} (Middle), respectively. The results on Qwen3 for K\&K tasks are shown in Figure~\ref{fig:qwen_pre2_kk} (Left). Across different models and tasks, we observe consistent conclusions: the Oracle strategy substantially alleviates collapse, validating the effectiveness of the CAG metric for mitigating collapse in unsupervised training.

\begin{figure}[t]
    \centering
    \begin{minipage}[t]{0.16\linewidth}
        \centering
        \includegraphics[width=\linewidth]{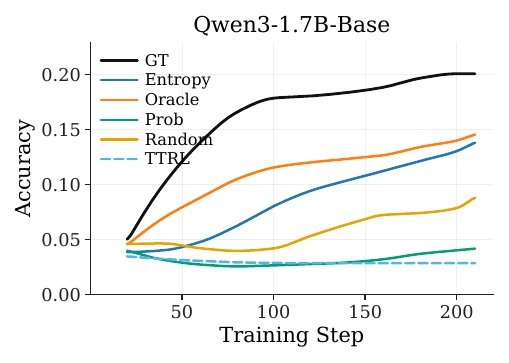}
    \end{minipage}
    \begin{minipage}[t]{0.16\linewidth}
        \centering
        \includegraphics[width=\linewidth]{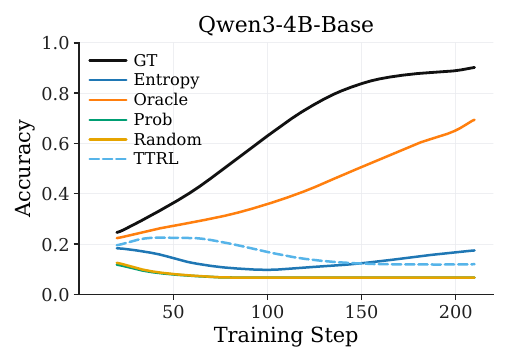}
    \end{minipage}
    \begin{minipage}[t]{0.16\linewidth}
        \centering
        \includegraphics[width=\linewidth]{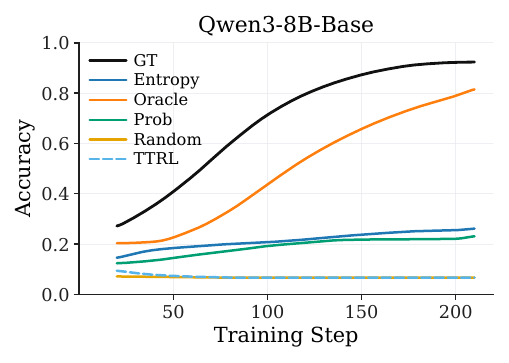}
    \end{minipage}
    \begin{minipage}[t]{0.16\linewidth}
        \centering
        \includegraphics[width=\linewidth]{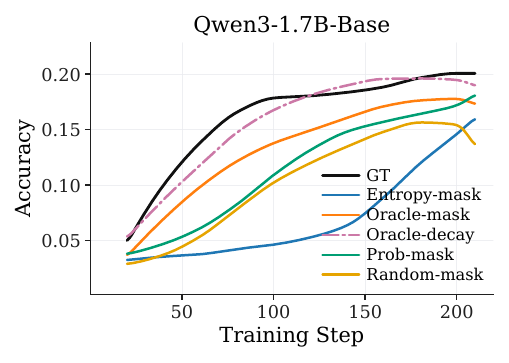}
    \end{minipage}
    \begin{minipage}[t]{0.16\linewidth}
        \centering
        \includegraphics[width=\linewidth]{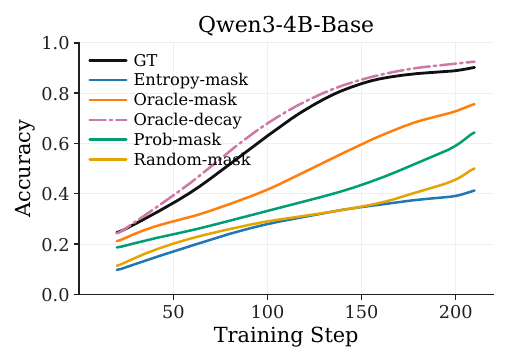}
    \end{minipage}
    \begin{minipage}[t]{0.16\linewidth}
        \centering
        \includegraphics[width=\linewidth]{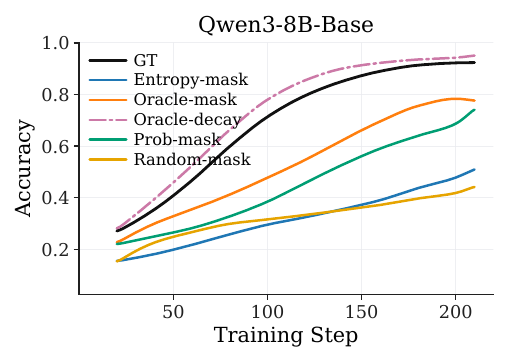}
    \end{minipage}
    \caption{Comparison of label acquisition strategies on K\&K pseudo-labeled samples using the Qwen3 model without masking \textbf{(Left)} and with masking \textbf{(Right)}.}
    \label{fig:qwen_pre2_kk}
\end{figure}

We further compared the effectiveness of different methods under the condition that only samples with ground-truth labels are used, with other samples masked out. Figure~\ref{fig:annotation_value_comparison_mask} and Figure~\ref{fig:phi_pre_exp} (Right) report the results on Qwen3 and Phi4-mini-instruct on math datasets, respectively. Figure~\ref{fig:qwen_pre2_kk} (Right) and Figure~\ref{fig:phi_pre_exp_kk} (Right) report the results on Qwen3 and Phi4-mini-instruct on K\&K datasets, respectively. In the absence of KL constraints, the Phi model shows occasional instability, and both Oracle-decay and GRPO tend to collapse during training, but the optimal performance of Oracle-decay remains markedly superior to that of other approaches. The results show a generally consistent trend: Oracle-mask performs favorably in most cases, providing empirical support for using CAG as a criterion for supervision value selection. We also experimented with using the CAG metric to decay the advantage:

\begin{equation}
\bar{A}_{i,g} = A^{\star}_{i,g} \times \mathrm{exp}(-100 \cdot s_i)
\end{equation}

denoted as Oracle-decay. With this large decay coefficient, the advantage of incorrect samples approaches 0, while for unsupervised samples, where the voting is correct, CAG is 0, and the advantage scaling factor is 1, thus integrating correct unsupervised samples into training. This approach achieves performance comparable to GT in most models, further demonstrating the value of properly utilizing unsupervised samples.

\begin{figure}[t]
    \centering
    \begin{minipage}[t]{0.28\linewidth}
        \centering
        \includegraphics[width=\linewidth]{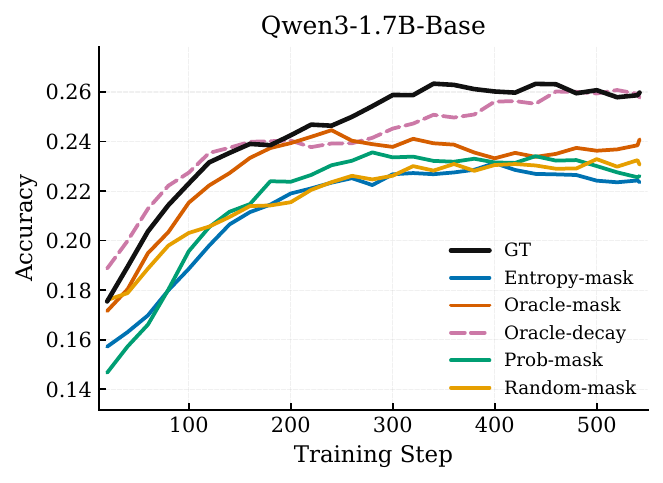}
    \end{minipage}
    \begin{minipage}[t]{0.28\linewidth}
        \centering
        \includegraphics[width=\linewidth]{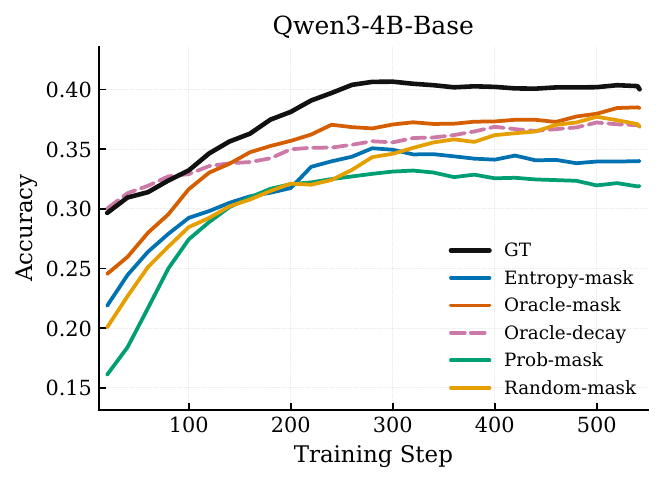}
    \end{minipage}
    \begin{minipage}[t]{0.28\linewidth}
        \centering
        \includegraphics[width=\linewidth]{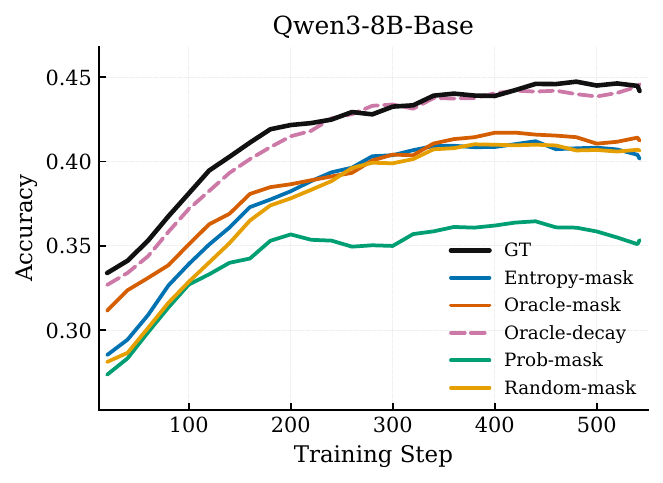}
    \end{minipage}
    \caption{Comparison of label acquisition strategies across three Qwen3 model scales on math tasks.}
    \label{fig:annotation_value_comparison_mask}
\end{figure}

\subsection{More Ablation Study Results}
\label{appendix:more_ablation_results}
We provide ablation study results for the Qwen3-1.7B-Base and Qwen3-8B-Base models on math tasks, as shown in Figure~\ref{fig:analysis_exp2} (left and middle).

\begin{figure}[t]
    \centering
    \begin{minipage}[t]{0.28\linewidth}
        \centering
        \includegraphics[width=\linewidth]{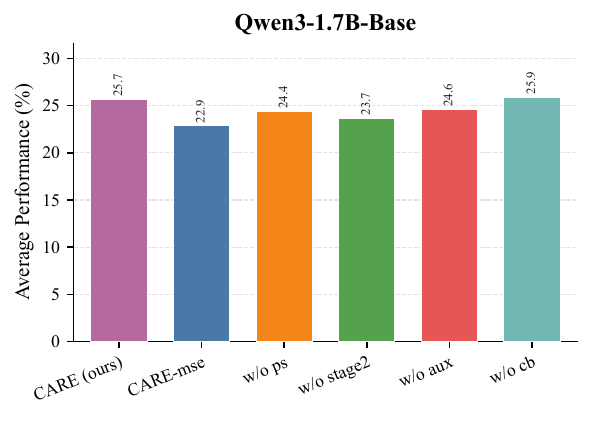}
    \end{minipage}
    \begin{minipage}[t]{0.28\linewidth}
        \centering
        \includegraphics[width=\linewidth]{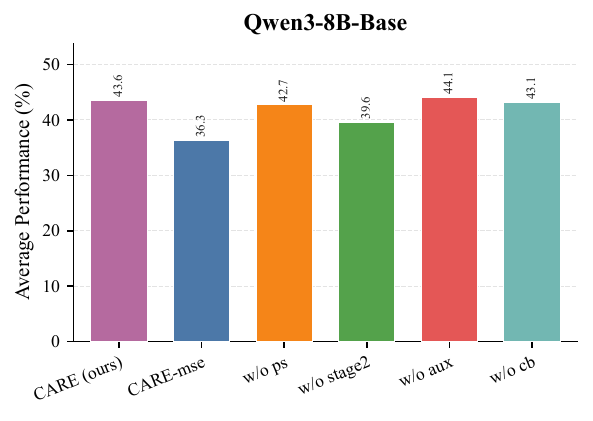}
    \end{minipage}
    \begin{minipage}[t]{0.35\linewidth}
        \centering
        \includegraphics[width=\linewidth]{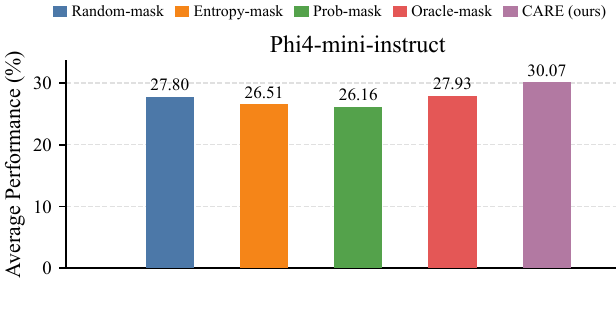}
    \end{minipage}
    \caption{\textbf{(Left and Middle)} CARE component ablation on Qwen3-1.7B-Base and Qwen3-8B-Base on math tasks. \textbf{(Right)} Comparison of training performance without unsupervised samples across baseline methods on Phi4-mini-instruct on math tasks.}
    \label{fig:analysis_exp2}
\end{figure}

\subsection{Additional Analysis of $p_2$ on More Models}
\label{app:p2_analysis}
As shown in Figure~\ref{fig:more_p2_analysis}, for Qwen3-1.7B-Base, the first-stage classifier achieves lower accuracy, likely due to the model's smaller hidden dimension and less stable training. In this case, a larger $p_2$ introduces more erroneous pseudo-labeled samples, while $p_2=0.1$ provides a better accuracy--performance trade-off. For other models, the larger hidden dimensions and more stable training improve classifier learning, making the classifier accuracy less sensitive to $p_2$. In these cases, increasing $p_2$ introduces more unsupervised samples and improves performance, with $p_2=0.25$ yielding the best results.

\begin{figure}[t]
    \centering
    \begin{minipage}[t]{0.66\linewidth}
        \centering
        \includegraphics[width=\linewidth]{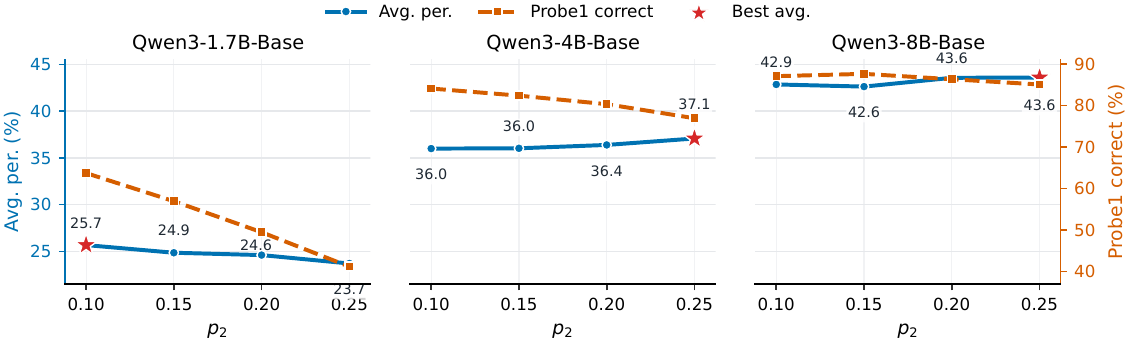}
    \end{minipage}
    \begin{minipage}[t]{0.32\linewidth}
        \centering
        \includegraphics[width=\linewidth]{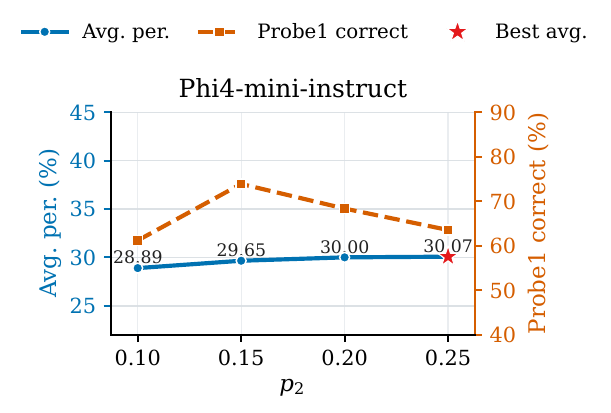}
    \end{minipage}
    \caption{Hyperparameter analysis of $p_2$ on Qwen3-Base (left) and Phi4 (right) models on math tasks.}
    \label{fig:more_p2_analysis}
\end{figure}

\subsection{Main Experimental Results on Math Tasks Using the Phi4 Model}
\label{appendix:phi4_main_exp}
To further examine the effectiveness of CARE across different model families on math tasks, we conduct experiments on Phi4-mini-instruct, following the same hyperparameter settings as those described in Appendix~\ref{appendix:implementation_details}. The experimental results on math tasks are shown in Table~\ref{tab:phi-main-results} and Figure~\ref{fig:analysis_exp2} (right). CARE consistently outperforms other data selection methods in both settings, i.e., with and without pseudo-labeled samples, and even achieves performance comparable to GT. These results further demonstrate the effectiveness of CARE on math tasks.

\begin{table}[t]
\centering
\small
\setlength{\tabcolsep}{3pt} 
\renewcommand{\arraystretch}{1.2}
\caption{Performance (\%) of Phi4-mini-instruct models on math tasks. Vanilla and GT are reported as reference settings and are excluded from best/second-best ranking. The \textbf{bold} and \underline{underline} indicate the best and second-best results among the remaining methods, respectively.}

\resizebox{0.7\linewidth}{!}{%
\begin{tabular}{lccccccc}
\toprule
\textbf{Methods} & \textbf{AIME24} & \textbf{AIME25} & \textbf{MATH500} & \textbf{AMC23} & \textbf{Hmmt25} & \textbf{Olympiad} & \textbf{Average} \\
\hline
\rowcolor{mygray}
\multicolumn{8}{l}{\textbf{\emph{Phi4-mini-instruct}}} \\
\rowcolor{gray!8}
Vanilla & 6.46 & 3.33 & 39.93 & 19.30 & 0.52 & 24.11 & 15.61 \\
\rowcolor{gray!8}
GT  & 11.56 & 9.58 & 73.25 & 47.26 & 0.63 & 38.07 & 30.06 \\
TTRL & 8.44 & \underline{6.67} & 66.30 & 42.34 & 0.63 & 33.55 & 26.32 \\
Random & 7.71 & 5.83 & 67.70 & 41.64 & 0.73 & 33.77 & 26.23 \\
Entropy & 8.65 & 5.83 & 63.95 & 40.63 & 0.73 & 33.10 & 25.48 \\
Prob & 8.02 & 4.90 & 67.83 & 42.34 & \underline{1.25} & 33.14 & 26.25 \\
Oracle & \underline{10.93} & 5.94 & \underline{69.73} & \underline{42.97} & 0.94 & \underline{35.12} & \underline{27.60} \\
\rowcolor{yellow!20}
\textbf{CARE (ours)} & \textbf{12.92} & \textbf{7.81} & \textbf{73.77} & \textbf{47.66} & \textbf{1.35} & \textbf{36.90} & \textbf{30.07} \\

\bottomrule
\end{tabular}%
}

\label{tab:phi-main-results}
\end{table}

\subsection{Main Experimental Results on K\&K Tasks}
\label{appendix:kk_main_exp}
To further verify the effectiveness of CARE on other tasks, we evaluate it in the logical reasoning domain by conducting experiments on the K\&K dataset~\cite{xie2025memorization}, following the same hyperparameter settings as those described in Appendix~\ref{appendix:implementation_details}. We use the 3--7 ppl split for training and the 3--8 ppl split for testing. The dataset statistics are shown in Table~\ref{tab:datasets_statics}. The experimental results of the Qwen3-Base and Phi4-mini-instruct models on the K\&K dataset are shown in Table~\ref{tab:kk-main-results} and Figure~\ref{fig:main-kk-mask-exp-qwen_phi}. CARE consistently outperforms other data selection methods in both settings, i.e., with and without pseudo-labeled samples, validating the effectiveness of CARE on K\&K tasks.

\begin{table}[t]
\centering
\small
\setlength{\tabcolsep}{3pt} 
\renewcommand{\arraystretch}{1.2}
\caption{Performance (\%) of Qwen3-Base and Phi4-mini-instruct models on K\&K tasks. Vanilla and GT are reported as reference settings and are excluded from best/second-best ranking. The \textbf{bold} and \underline{underline} indicate the best and second-best results among the remaining methods, respectively.}

\resizebox{0.7\linewidth}{!}{%
\begin{tabular}{lccccccc}
\toprule
\textbf{Methods} & \textbf{3ppl} & \textbf{4ppl} & \textbf{5ppl} & \textbf{6ppl} & \textbf{7ppl} & \textbf{8ppl} & \textbf{Average} \\
\hline
\rowcolor{mygray}
\multicolumn{8}{l}{\textbf{\emph{Qwen3-1.7B-Base}}} \\
\rowcolor{gray!8}
Vanilla & 2.13 & 1.50 & 0.37 & 0.75 & 0.12 & 0.00 & 0.81 \\
\rowcolor{gray!8}
GT  & 40.25 & 30.13 & 19.00 & 20.87 & 6.75 & 8.12 & 20.85 \\
TTRL & 10.63 & 4.87 & 1.88 & 1.75 & 1.25 & 1.25 & 3.60 \\
Random & 18.25 & 10.00 & 5.00 & 4.25 & 0.88 & 0.00 & 6.40 \\
Entropy & 13.12 & 2.88 & 5.12 & 2.75 & 0.63 & 0.00 & 4.08 \\
Prob & 11.63 & 2.88 & 5.12 & 2.75 & 0.63 & 0.00 & 3.60 \\
Oracle & \underline{27.38} & \underline{20.87} & \underline{15.75} & \underline{15.25} & \underline{4.00} & \underline{5.00} & \underline{14.71} \\
\rowcolor{yellow!20}
\textbf{CARE (ours)} & \textbf{41.37} & \textbf{30.50} & \textbf{17.75} & \textbf{18.63} & \textbf{6.50} & \textbf{11.34} & \textbf{21.02} \\
\hline

\rowcolor{mygray}
\multicolumn{8}{l}{\textbf{\emph{Qwen3-4B-Base}}} \\
\rowcolor{gray!8}
Vanilla & 9.25 & 7.63 & 4.63 & 3.00 & 1.75 & 0.63 & 4.48 \\
\rowcolor{gray!8}
GT  & 99.37 & 98.12 & 95.25 & 88.38 & 85.12 & 74.75 & 90.17 \\
TTRL & 50.25 & 35.37 & 28.13 & 21.88 & 12.87 & 9.25 & 26.29 \\
Random & 56.63 & 43.00 & 28.75 & 24.38 & 13.88 & 10.88 & 29.58 \\
Entropy & 28.25 & 16.75 & 14.12 & 10.50 & 3.00 & 4.13 & 12.79 \\
Prob & 28.37 & 18.00 & 11.87 & 11.12 & 2.62 & 4.87 & 12.81 \\
Oracle & \underline{89.00} & \underline{87.25} & \underline{74.87} & \underline{63.38} & \underline{59.25} & \underline{54.37} & \underline{71.42} \\
\rowcolor{yellow!20}
\textbf{CARE (ours)} & \textbf{97.75} & \textbf{96.25} & \textbf{89.50} & \textbf{83.13} & \textbf{76.75} & \textbf{68.62} & \textbf{85.33} \\
\hline

\rowcolor{mygray}
\multicolumn{8}{l}{\textbf{\emph{Qwen3-8B-Base}}} \\
\rowcolor{gray!8}
Vanilla & 25.25 & 17.50 & 13.88 & 8.75 & 5.50 & 3.62 & 12.42\\
\rowcolor{gray!8}
GT & 99.00 & 97.87 & 97.00 & 89.00 & 88.75 & 84.62 & 92.71 \\
TTRL & 24.88 & 10.37 & 6.75 & 8.12 & 2.00 & 4.00 & 9.35 \\
Random & 25.75 & 12.13 & 8.25 & 10.63 & 2.37 & 4.13 & 10.54 \\
Entropy & 46.00 & 41.00 & 32.62 & 32.00 & 14.00 & 13.00 & 29.77 \\
Prob & 41.12 & 30.38 & 11.20 & 14.25 & 5.00 & 4.00 & 17.79 \\
Oracle & \underline{95.00} & \underline{93.50} & \underline{86.75} & \underline{78.50} & \underline{74.25} & \underline{69.63} & \underline{82.94} \\
\rowcolor{yellow!20}
\textbf{CARE (ours)} & \textbf{99.37} & \textbf{98.00} & \textbf{93.25} & \textbf{91.00} & \textbf{86.75} & \textbf{77.38} & \textbf{90.96} \\
\hline

\rowcolor{mygray}
\multicolumn{8}{l}{\textbf{\emph{Phi4-mini-instruct}}} \\
\rowcolor{gray!8}
Vanilla & 21.88 & 16.00 & 8.00 & 6.37 & 3.00 & 1.75 & 9.50 \\
\rowcolor{gray!8}
GT & 64.75 & 54.12 & 41.75 & 38.12 & 27.88 & 21.00 & 41.27 \\
TTRL & 46.88 & 33.63 & 21.63 & 20.37 & 13.50 & 10.13 & 24.35 \\
Random & 45.50 & 31.87 & 19.00 & 18.50 & 13.12 & 9.13 & 22.85 \\
Entropy & 48.25 & 35.00 & 26.25 & 23.25 & 14.25 & 10.88 & 26.31 \\
Prob & 47.37 & 33.88 & 24.25 & 20.50 & 13.38 & 12.25 & 25.27 \\
Oracle & \underline{54.25} & \underline{44.50} & \underline{32.88} & \underline{33.50} & \underline{23.50} & \underline{17.25} & \underline{34.31} \\
\rowcolor{yellow!20}
\textbf{CARE (ours)} & \textbf{64.50} & \textbf{54.62} & \textbf{37.38} & \textbf{36.00} & \textbf{25.50} & \textbf{22.25} & \textbf{40.04} \\
\bottomrule
\end{tabular}%
}

\label{tab:kk-main-results}
\end{table}

\begin{figure}[t]
    \centering
    \begin{minipage}[t]{0.24\linewidth}
        \centering
        \includegraphics[width=\linewidth]{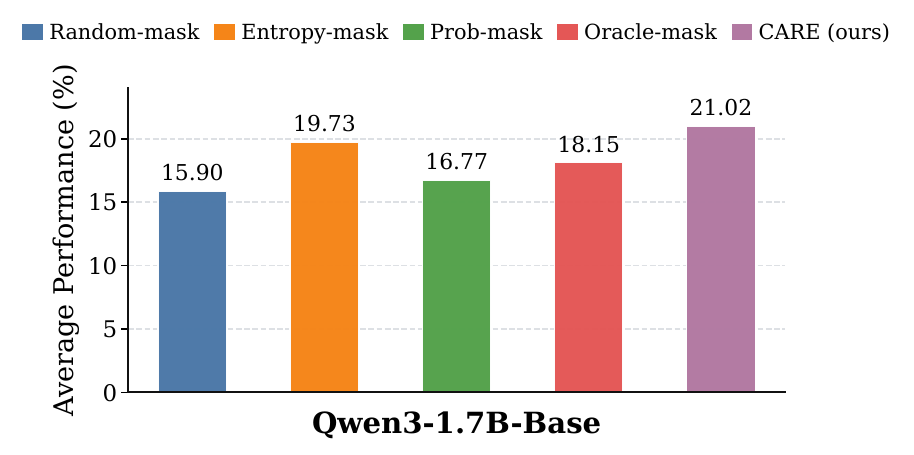}
    \end{minipage}
    \begin{minipage}[t]{0.24\linewidth}
        \centering
        \includegraphics[width=\linewidth]{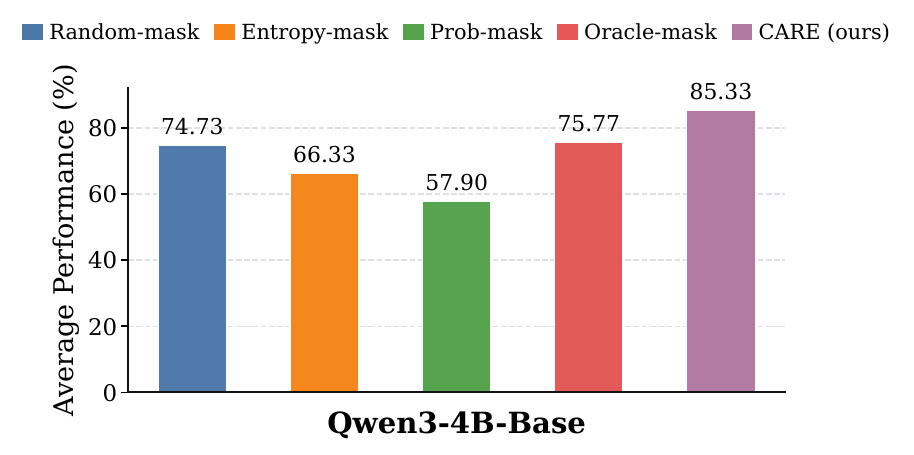}
    \end{minipage}
    \begin{minipage}[t]{0.24\linewidth}
        \centering
        \includegraphics[width=\linewidth]{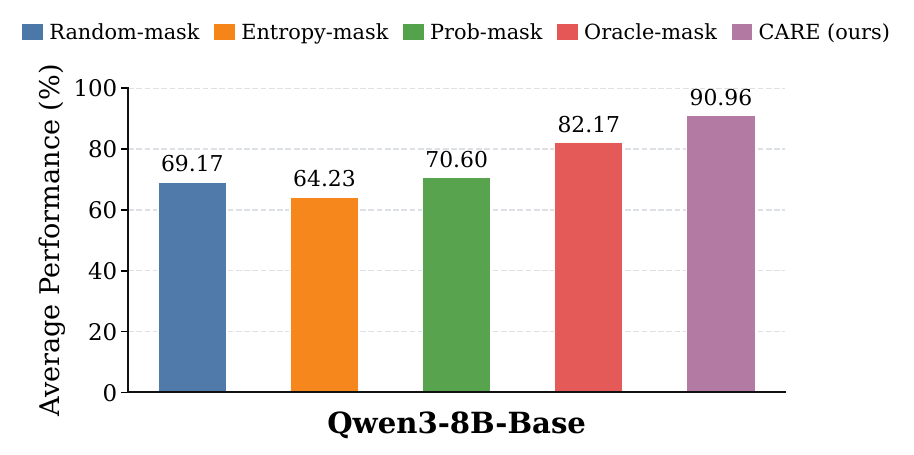}
    \end{minipage}
    \begin{minipage}[t]{0.24\linewidth}
        \centering
        \includegraphics[width=\linewidth]{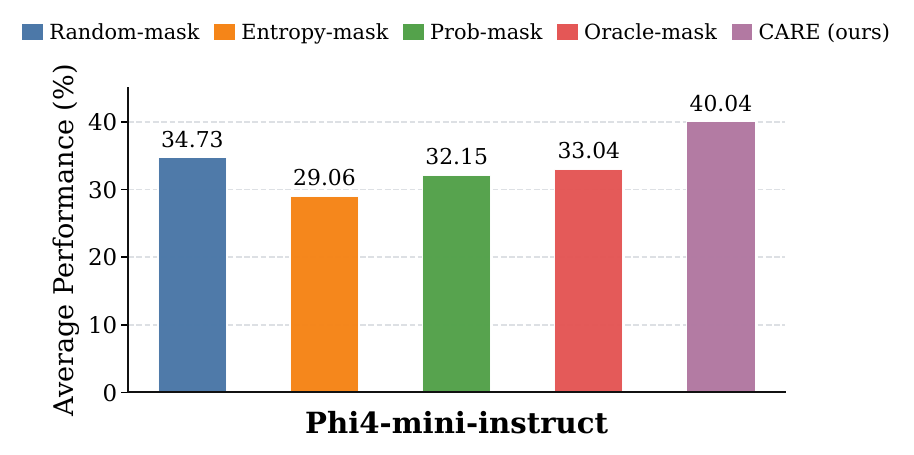}
    \end{minipage}
    \caption{Comparison of training performance among baselines without unsupervised samples on K\&K tasks using Qwen3-Base and Phi4-mini-instruct models.}
    \label{fig:main-kk-mask-exp-qwen_phi}
\end{figure}

\section{Pseudocode of Class-wise CAG Score Calculation}
We present the pseudocode of Class-wise CAG Score Calculation in Algorithm~\ref{alg:cas_class}.
\begin{algorithm}[t]
\caption{Class-wise CAG Score Calculation}
\label{alg:cas_class}
\begin{algorithmic}[1]
\REQUIRE Voted-cluster size $m_i$, admissible class $k\in\mathcal{K}_i$, number of rollouts $G$
\ENSURE Class-wise CAG score $s_i(k)$
\STATE Construct the pseudo reward induced by majority voting:
\[
\tilde{\mathbf r}_i=
[
\underbrace{1,\ldots,1}_{m_i},
\underbrace{0,\ldots,0}_{G-m_i}
]
\]
\STATE Construct the candidate ground-truth reward under class $k$:
\[
\mathbf r_i^{(k)}=
[
\underbrace{0,\ldots,0}_{m_i},
\underbrace{1,\ldots,1}_{k},
\underbrace{0,\ldots,0}_{G-m_i-k}
]
\]
\STATE Compute the pseudo-label advantage $\tilde{\mathbf A}_i$ from $\tilde{\mathbf r}_i$ using Eq.~\eqref{eq:grpo_adv}
\STATE Compute the candidate ground-truth advantage $\mathbf A_i^{(k)}$ from $\mathbf r_i^{(k)}$ using Eq.~\eqref{eq:grpo_adv}
\STATE Compute the class-wise CAG score $s_i(k)$ using Eq.~\eqref{eq:cag}
\RETURN $s_i(k)$
\end{algorithmic}
\end{algorithm}

\section{Pseudocode of CARE}
We present the pseudocode of CARE in Algorithm~\ref{alg:care}.

\begin{algorithm}[t]
\caption{CARE: Correction-Aware Reliability Estimation for RLAVR}
\label{alg:care}
\begin{algorithmic}[1]
\STATE \textbf{Input:} initial policy $\pi_{\theta_{\mathrm{init}}}$; training prompts $D_{\mathrm{train}}$; verifier; annotation ratio $p$, selection ratio of unsupervised samples $p_2$; warmup steps $T_{\mathrm{warm}}$; GRPO steps $\mu$; replay buffer $\mathcal{R}$; first-stage classifier $g_{\psi_1}$; second-stage classifier $g_{\psi_2}$
\STATE \textbf{Output:} policy model $\pi_\theta$
\STATE Initialize policy $\pi_\theta \leftarrow \pi_{\theta_{\mathrm{init}}}$
\STATE Initialize replay buffer $\mathcal{R} \leftarrow \emptyset$
\FOR{iteration $=1,\ldots,I$}
    \STATE Set reference model $\pi_{\mathrm{ref}} \leftarrow \pi_\theta$
    \FOR{step $=1,\ldots,M$}
        \STATE Sample a batch of prompts $D_b$ from $D_{\mathrm{train}}$
        \STATE Update old policy $\pi_{\theta_{\mathrm{old}}} \leftarrow \pi_\theta$
        \STATE Sample $G$ rollouts $\mathcal{B}_i=\{o_{i,1},\ldots,o_{i,G}\}\sim \pi_{\theta_{\mathrm{old}}}(\cdot\mid x_i)$ for each $x_i\in D_b$
        \STATE Compute classifier representation $\phi_i=[\phi_i^{\mathrm{global}};\Phi_i^{\mathrm{ans}}]$ for each $x_i\in D_b$
        \IF{step $\leq T_{\mathrm{warm}}$}
            \STATE Set $\mathcal{S}_{\mathrm{err}}\leftarrow D_b$ and $\mathcal{S}_{\mathrm{unsup}}\leftarrow \emptyset$
        \ELSE
            \STATE Use $g_{\psi_1}$ to predict the vote reliability score $c_i^{(1)}$ for each $x_i\in D_b$
            \STATE Select the top-$p_2\%$ prompts with the largest $c_i^{(1)}$ as $\mathcal{S}_{\mathrm{unsup}}$
            \STATE Let $\mathcal{S}_{\mathrm{err}}\leftarrow D_b\setminus \mathcal{S}_{\mathrm{unsup}}$
        \ENDIF
        \STATE Compute CAG scores $\{\hat{s}_i\}_{x_i\in\mathcal{S}_{\mathrm{err}}}$ using $g_{\psi_2}$ and Algorithm~\ref{alg:cas_class}
        \STATE Select the top-$p\%$ prompts in $\mathcal{S}_{\mathrm{err}}$ with the largest $\hat{s}_i$ as $\mathcal{S}_{\mathrm{sup}}$
        \STATE Acquire ground-truth labels for prompts in $\mathcal{S}_{\mathrm{sup}}$
        \STATE Compute ground-truth rewards and advantages $\{A_{i,g}^{\star}\}_{j=1}^{n}$ for each $x_i\in\mathcal{S}_{\mathrm{sup}}$
        \STATE Compute pseudo rewards and pseudo advantages $\{\tilde A_{i,g}\}_{j=1}^{n}$ for each $x_i\in\mathcal{S}_{\mathrm{unsup}}$
        \STATE Use the reliability-weighted pseudo advantages $\{\bar A_{i,g}\}_{j=1}^{n}$ for each $x_i\in\mathcal{S}_{\mathrm{unsup}}$
        \STATE Construct mixed advantages $A_{i,g}^{\mathrm{mix}}$ using $A_{i,g}^{\star}$ for $\mathcal{S}_{\mathrm{sup}}$ and $\bar A_{i,g}$ for $\mathcal{S}_{\mathrm{unsup}}$
        \FOR{GRPO update $=1,\ldots,\mu$}
            \STATE Update $\pi_\theta$ by optimizing the GRPO objective in Eq.~\eqref{eq:grpo_obj} with $A_{i,g}^{\mathrm{mix}}$
        \ENDFOR
        \STATE Add newly labeled samples from $\mathcal{S}_{\mathrm{sup}}$ to the replay buffer $\mathcal{R}$
        \STATE Update $g_{\psi_1}$ and $g_{\psi_2}$ using current labeled samples and replayed samples from $\mathcal{R}$
    \ENDFOR
\ENDFOR
\STATE \textbf{return} $\pi_\theta$
\end{algorithmic}
\end{algorithm}

\section{Limitations}
\label{sec:limitations}
Our work has several limitations. Due to limited computational resources, our experiments do not cover larger-scale models, with the largest evaluated model being 8B. In addition, the current evaluation mainly focuses on text-only mathematical and logical reasoning tasks, leaving broader task domains and multimodal scenarios unexplored. Moreover, some hyperparameters in CARE still rely on manual tuning. In future work, we plan to validate CARE on larger models and more diverse RLAVR scenarios, including code generation, tool use, agentic tasks, and multimodal tasks. We also aim to make the parameter selection process adaptive to task characteristics and model capabilities.

\section{Broader Impacts}
\label{sec:broader_impacts}
This work introduces the new settings RLAVR, where acquiring ground-truth labels is expensive. By improving the use of limited supervision, CARE may help improve data efficiency and reduce annotation costs in domains where labels are difficult to obtain. The potential negative impact is limited, as our work focuses on methodological improvements rather than specific downstream applications. However, as with other advances in LLMs, improved training methods may indirectly enhance generative capabilities, which could be misused for generating misleading or harmful content. These risks are general to the field and not specific to our approach.

\section{Use of Large Language Models}
Some parts of the text were refined with the assistance of large language models (LLMs). All content and responsibility for the work remain with the authors.

\section{Case Study}
\label{appendix:case_study}
We present an example of the outputs generated by Qwen3-4B-Base trained with CARE on DAPO-17k training set in Figure~\ref{fig:4b_care_case1}, where the model correctly answers the question. We further provide a CAG calculation example in Figure~\ref{fig:cag_calculate_case}.

\begin{figure*}[htbp]
    \centering
    \includegraphics[width=1\textwidth]{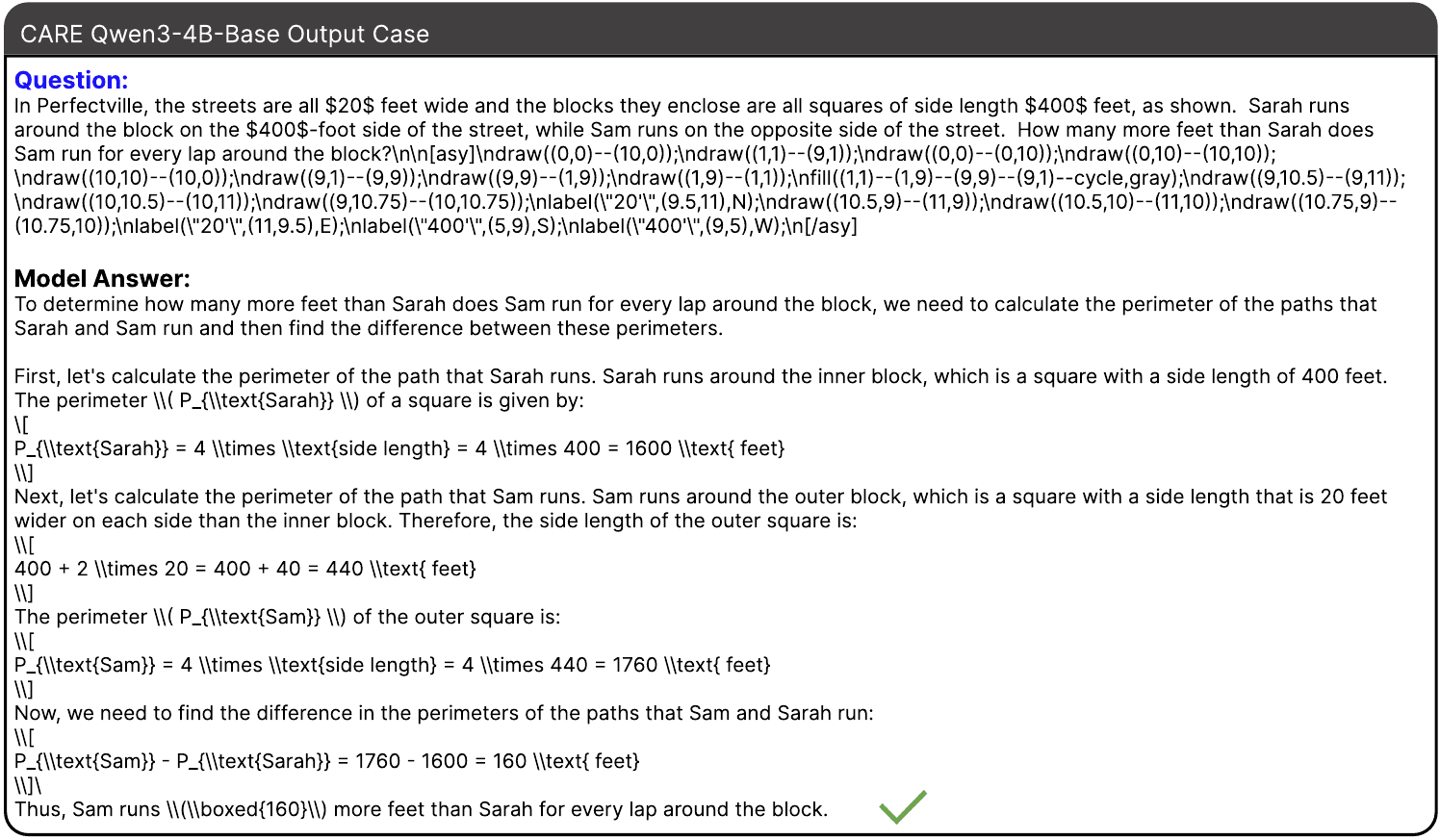}
    \caption{Example 1 of the CARE method using the Qwen3-4B-Base model.}
    \label{fig:4b_care_case1}
\end{figure*}

\begin{figure*}[htbp]
    \centering
    \includegraphics[width=1\textwidth]{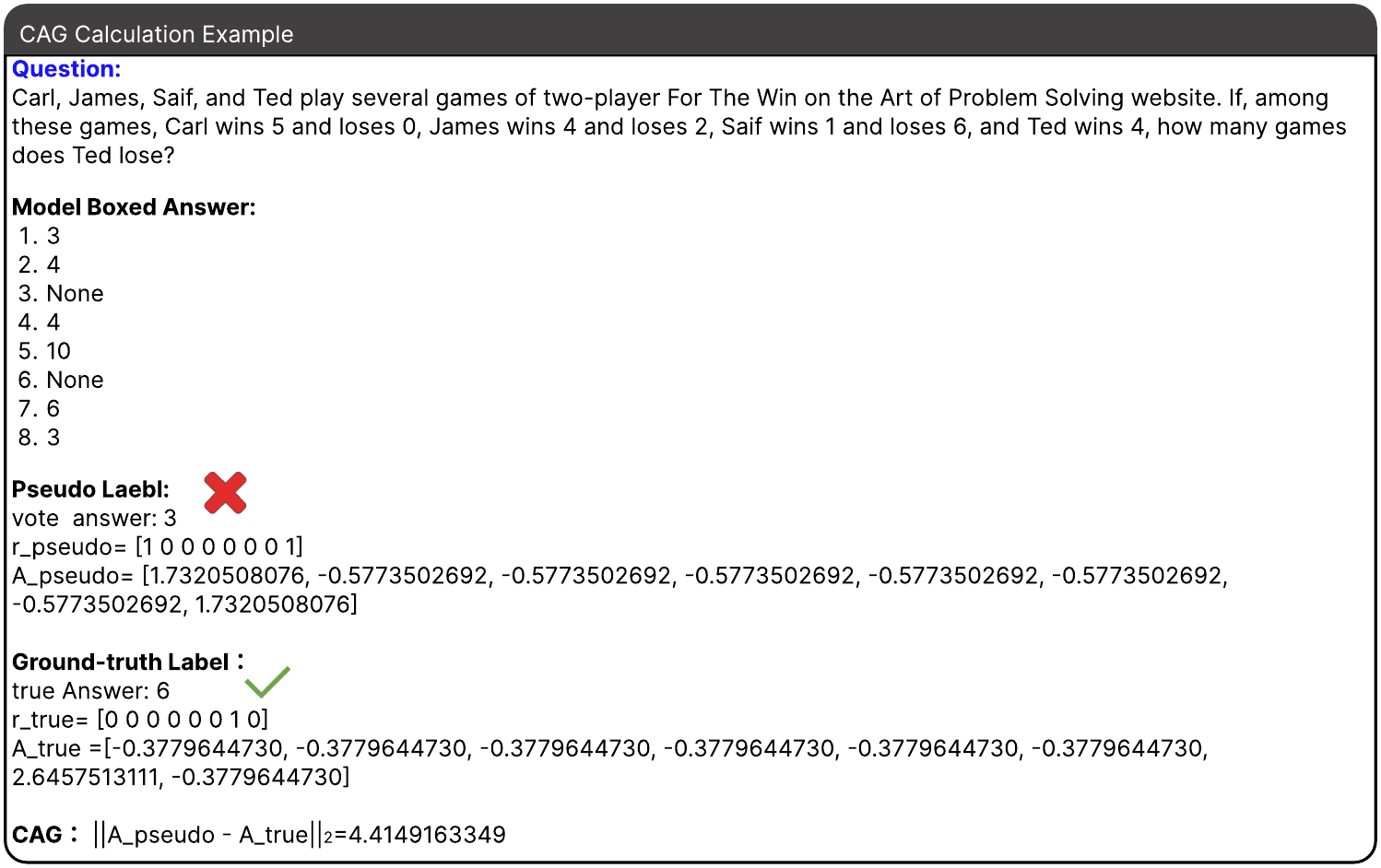}
    \caption{CAG calculation example.}
    \label{fig:cag_calculate_case}
\end{figure*}

\end{document}